\documentclass{article}
\pdfoutput=1

\usepackage[final,arxiv,nonatbib]{nips_2017}

\usepackage[utf8]{inputenc} 
\usepackage[T1]{fontenc}    
\usepackage{hyperref}       
\usepackage{url}            
\usepackage{booktabs}       
\usepackage{amsfonts}       
\usepackage{nicefrac}       
\usepackage{microtype}      

\usepackage{amsmath}
\usepackage{amssymb}
\usepackage{amsthm}
\usepackage{pgffor}
\usepackage{xparse}

\usepackage{algorithm}
\usepackage{algorithmic}

\usepackage{graphicx} 

\usepackage{subcaption}

\def\datasetsa{mnist,svhn,cifar10,celeba64}

\newcommand{\y}{{\mathbf y}}
\newcommand{\z}{{\mathbf z}}
\newcommand{\x}{{\mathbf x}}

\renewcommand{\Re}{{\mathbb R}}

\usepackage{todonotes}

\newtheorem{claim}{Claim}
\newtheorem{proposition}{Proposition}

\title{Generator Reversal}

\author{
    Yannic Kilcher \\
  ETH Zurich, Switzerland\\
  \texttt{yannic.kilcher@inf.ethz.ch} \\
  \And
    Aur\'elien Lucchi \\
  ETH Zurich, Switzerland\\
  \texttt{aurelien.lucchi@inf.ethz.ch} \\
  \And
    Thomas Hofmann \\
  ETH Zurich, Switzerland\\
  \texttt{thomas.hofmann@inf.ethz.ch} \\
}

\begin{document}

\maketitle

\begin{abstract}
We consider the problem of training generative models with deep neural networks as generators, i.e.~to map latent codes to data points. Whereas the dominant paradigm combines simple priors over codes with complex deterministic models, we propose instead to use more flexible code distributions. These distributions are estimated non-parametrically by reversing the generator map during training. The benefits include: more powerful generative models, better modeling of latent structure and explicit control of the degree of generalization.
\end{abstract}


\section{Introduction}

Continuous latent variable models have been developed and studied in statistics for almost a century, with factor analysis \cite{young1941maximum, bartholomew1987factor} being the most paradigmatic and widespread model family.  In the neural network community autoencoders have been studied classically as  dimension reduction methods \cite{baldi1989neural,demers1993n}, which encode observations into low-dimensional codes allowing  approximate reconstruction via decoder or generator networks. By driving such a generator with randomness \cite{mackay1995bayesian}, one obtains a probabilistic generative model. This line of research has been further developed for deep autoencoders \cite{hinton2006reducing,vincent2010stacked} and belief networks \cite{hinton2006fast,bengio2007greedy}, making use of stacked or layerwise training with a focus on pre-training representations for supervised tasks.

However, generative models have recently moved to the center stage of deep learning in their own right. Most notable is the seminal work on Generative Adversarial Networks (GAN) \cite{Goodfellow:2014td} as well as probabilistic architectures known as Variational Autoencoder (VAE) \cite{Kingma:2013tz,Rezende:2014vm}. Here, the focus has moved away from density estimation and towards generative models that -- informally speaking -- produce samples that are perceptually indistinguishable from samples generated by nature. This is particularly relevant in the context of high-dimensional signals such as images, speech, or text.

Generative models like GANs, VAEs and others typically define a generative model via a deterministic generative mechanism or generator $G_\phi: \Re^k \to \Re^m$, $\z \mapsto G(\z) =\x$, parametrized by $\phi$. They are often implemented as a deep neural network (DNN), which is hooked up to a code distribution $\z \sim \mathcal P_\z$, to induce a distribution $\x \sim \mathcal P_\x$. It is known that under mild regularity conditions, by a suitable choice of generator, any $\mathcal P_\x$ can be obtained from an arbitrary \textit{fixed} $\mathcal P_\z$ \cite{kallenberg2006foundations}. Relying on the power and flexibility of DNNs, this has led to the view that code distributions should be simple and \textit{a priori} fixed, e.g.~$\mathcal P_\z = \mathcal N(\mathbf 0, \mathbf I)$. As shown in \cite{arjovsky2017towards}  for DNN generators, $\text{Im}(G_\phi)$ is a countable union of manifolds of dimension $k$ though, which may pose challenges, if $k<m$. Whereas a current line of research addresses this via alternative (non-MLE or KL-based) discrepancy measures between distributions \cite{Dziugaite:2015wd, nowozin2016f,arjovsky2017wasserstein}, we investigate an orthogonal direction:
\begin{claim}
\label{claim:1}
It is advantageous to increase the modeling power of a generative model, not only via $G_\phi$, but by using more flexible prior code distributions $\mathcal P_\z$.  
\end{claim}
Another benefit of this approach is the ability to reveal richer structure (e.g.~multimodality) in latent space via $\mathcal P_\z$, a view which is also supported by evidence on using more powerful posterior distributions  \cite{mescheder2017adversarial}. 

Deep generative models have proven to be notoriously hard to train. The use of complementing recognition networks that operate in a bottom-up fashion and provide approximations to posteriors, has been a common trait of many models from the Helmholtz machine \cite{dayan1995helmholtz} to recent VAEs. In GANs this is avoided by pairing-up the generator with a discriminator in a minimax game. However, GANs can be brittle to train \cite{Radford:2015wf} and the training signal provided by the discriminator can become weak (\textit{vanishing gradients}) or misleading (\textit{mode collapse}). We would like to retain benefits of both approaches and propose an alternative method that does not make use of a recognition network and instead  relies on the generator itself to compute an approximate inverse map $H: \Re^m \to \Re^k$ such that $H \circ G_\phi \approx \text{id}$. 
\begin{claim}
The generator $G_\phi$ implicitly defines an approximate inverse, which can be computed with reasonable effort using gradient descent and without the need to co-train a recognition network. We call this approach generator reversal.
\label{claim:2}
\end{claim}
Our generator reversal is similar in spirit to~\cite{kindermann1990inversion}, but their intent differs as they use this technique as a tool to visualize the information processing capability of a neural network. Unlike previous works that require the transfer function to be bijective \cite{Baird:jr, Rippel:2013uq}, our approach does not strictly have this requirement, although this could still be imposed by carefully selecting the architecture of the network as shown in~\cite{dinh2016density, arjovsky2017wasserstein}.

Note that, if the above argument holds, we can easily find latent vectors $\z=H(\x)$ corresponding to given observations $\x$.  This then induces an empirical distribution of "natural" codes, which we can combine with the first argument.
\begin{claim}
Using generator reversal we can model natural code distributions in a non-parametric manner, for instance, using kernel density estimation. This can then be incorporated  into an improved learning method for GANs.
\label{claim:3}
\end{claim}
We have presented our main ideas as claims above and the rest of the paper will make these ideas specific and support them with evidence.


\section{Generator Reversal}
\label{sec:generator_reversal}

\subsection{Gradient--Based Reversal}

Let us begin with making Claim \ref{claim:2} more precise. Given a data point $\x$, we aim to compute some approximate code $H(\x)$ such that $(G \circ H)(\x) = G(H(\x)) =: \tilde \x \approx \x$. We do so by simple gradient descent, starting from some random initialization for $\z$ (see Algorithm \ref{alg:rgen}). We will come back to the question of how to define the termination condition.
\begin{algorithm}
    \caption{Generator Reversal}\label{alg:rgen}
    \begin{algorithmic}[1]
        \INPUT Data point $\x$, loss function $\ell$, initial value $\z_0$
        \STATE Initialize $\z \leftarrow \z_0$
        \REPEAT
        		\STATE $\hat{\x} = G_\phi(\z)$  \COMMENT{run generator}
	        \STATE $\z \leftarrow \z - \eta \nabla_\z \ell(\x,\tilde\x)$ \COMMENT{backpropagate error}
        \UNTIL{termination condition}
	\OUTPUT latent code $\z$
    \end{algorithmic}
\end{algorithm}

A key question is whether good (low loss) codevectors exist for data points $\x$. First of all, whenever $\x$ was actually generated by $G_\phi$, then surely we know that a perfect, zero-loss pre-image $\z$ exists. Of course finding it exactly would require the exact inverse function of the generator process but our experiments demonstrate that, in practice, an approximate answer is sufficient.

Secondly, if $\x$ is  in the training data, then as $G_\phi$ is trained to mimic the true distribution, it would be suboptimal if any such $\x$ would not have a suitable pre-image. We thus conjecture that learning a good generator will also improve the quality of the generator reversal, at least for points $\x$ of interest (generated or data).
Note that we do not explicitly train the generator to produce pre-images that would further optimize the training objective. This would require backpropagating through the reversal process which is certainly possible and would likely yield further improvements.

\subsection{Random Network Experiments}
\label{sec:random_network}

It is very hard to give quality guarantees for the approximations obtained via generator reversal. Here, we provide experimental evidence by showing that even a DNN generator with random weights $\phi$  can provide reasonable pre-images for data samples. As we argued above, we believe that actual training of $G_\phi$ will improve the quality of pre-images, so this is in a way a worst case scenario. 

Examples for three different image data sets are shown in Figure~\ref{fig:reconstruct:loss}. Here we show the average reconstruction error as a function of the number of gradient update steps. We observe that the error decreases steadily as the reconstruction progresses and reaches a low value very quickly. We also show randomly selected reconstructed samples in Figure~\ref{fig:reconstruct:img}, which reflect the fast decrease in terms of reconstruction error. After only 5 update steps, one can already recognize the general outline of the picture. This is perhaps even more surprising considering that these results were obtained using a generator with \emph{completely random weights}. A similar finding was also reported in~\cite{he2016powerful} which constructed deep visualizations using untrained neural networks initialized with random weights.

\def\steps{5, 20, 400}
\begin{figure}[t]
\vskip 0.2in
    \begin{minipage}{0.45\textwidth}
\begin{center}
\centerline{
    \includegraphics[width=1.2\columnwidth]{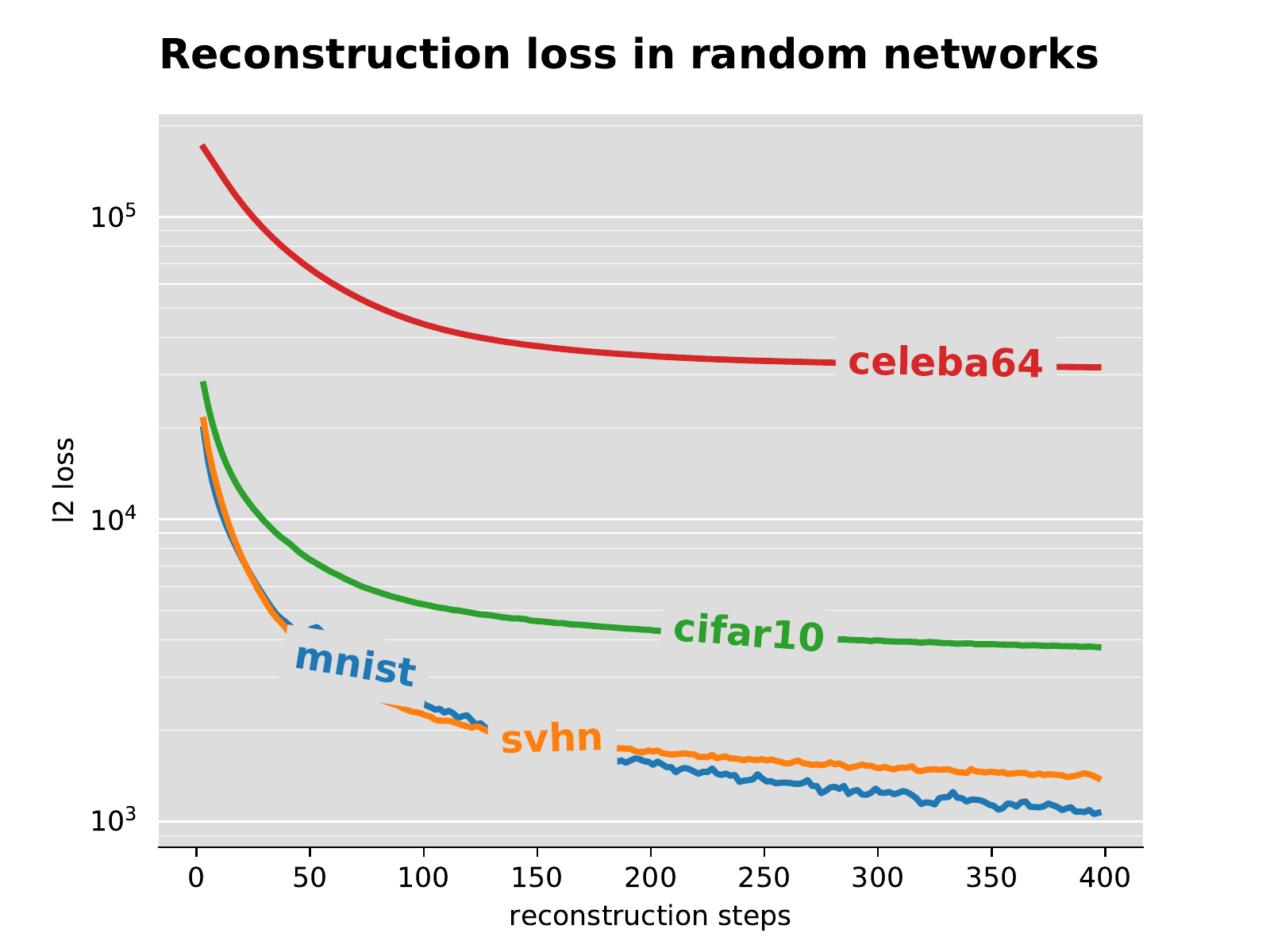}
}
\caption{\small{\it Reconstruction loss in generator networks with random weights.}}
\label{fig:reconstruct:loss}
\end{center}
    \end{minipage}
    \hspace*{1.2cm}
    \begin{minipage}{0.45\textwidth}
\begin{center}
    \vspace*{0.5cm}
\foreach \ds in \datasetsa
{
\centerline{
    \includegraphics[width=\columnwidth/5]{exp/reconstruct/reconstruct_\ds_orig}
    \foreach \stp in \steps
    {
        \includegraphics[width=\columnwidth/5]{exp/reconstruct/reconstruct_\ds_\stp}
    }
}
}
\caption{\small{\it Reconstruction quality using generator networks with random weights. The left column is the original image, followed by reconstructions after 5, 20 and 400 steps.}}
\label{fig:reconstruct:img}
\end{center}
    \end{minipage}
\vskip -0.2in
\end{figure} 

\subsection{Analysis} 

Let us provide a simple insight into generator reversal that ensures that locally, gradient descent will provide the correct pre-image.

\begin{proposition}
Let $G(\z)$ be locally invertible~\footnote{Note that this is a less restrictive assumption than the diffeomorphism property required in~\cite{arjovsky2017towards}}. For a point $\x = G(\z_*)$, the reconstruction problem with the $L_2$-loss is locally convex at $\z_*$.
\end{proposition}
\begin{proof}
Let $J_G(\z)$ denote the Jacobian of $G(\z)$. We prove the result stated above by showing that the Hessian at $\z_*$ is positive semidefinite.
\begin{align*}
  \ell(\z) 
  &= \frac{1}{2}\|G(\z) - \x\|^2 = \frac{1}{2}\|G(\z) - G(\z_*)\|^2 \\
\Longrightarrow   \nabla_\z \ell(\z) 
   &=  \mathbf J_G(\z) \left(G(\z) - G(\z_*)\right) \\
\Longrightarrow \nabla_\z^2 \ell(\z_*) 
   &=  0 \cdot \nabla_\z^2 G(\z) + \mathbf J_G(\z_*) \mathbf J_G(\z_*)^\top 
\end{align*}
Since $G(\z)$ is assumed to be locally invertible, then $\mathbf J_G(\z) \neq {\bf 0}$ and  the Hessian is therefore positive semidefinite.
\end{proof}

Note that one could also add an $\ell_2$ regularizer to obtain a locally strongly-convex function.


\section{Code Vector Distribution}

\subsection{Latent Code Structure}

With generator reversal at hand, we can now investigate the empirical distribution of code vectors for data samples. Basically, what we want to provide evidence for, is that a (GAN-)trained generator retains interesting structure in the latent code space, while this is not the case for a  generator with random weights. Moreover, we also want to stress, that there is far more structure in the latent space of the trained model than what an isotropic Gaussian is able to capture.  We use data with class labels to be able to assess how much of the class information is preserved in the latent codes.  Note, however, that we have not used the labels in any way during training. 

Results on MNIST are shown in Figure~\ref{fig:structure:mnist}, where we have projected latent representations from a trained network to a 2-dimensional space using t-SNE~\cite{maaten2008visualizing} and then colored according to their respective class label. Compared to an untrained network, there is a clear emergence of structure in the latent space which we can attribute to the GAN training procedure.

We believe these results to be paradigmatic in that they show that we should not assume that the latent space density always varies smoothly in all directions, nor that latent space directions are necessarily picking up on independent modes of variation. Even if we hard-wire such simplistic assumptions into the model (as we do in GANs), generator reversal reveals that the true code distribution of natural data can be quite different. This hints at a mismatch between modeling assumptions and natural data  distributions, which can be most directly overcome by increasing the model flexibility with regard to $\mathcal P_\z$.  By off-loading some of the modeling complexity from the mechanism $G_\phi$ to the driving distribution, the generator's challenge becomes easier, as anisotropic and clustered distributions do not have to be converted back into multivariate normal densities. This way, one can learn more meaningful representations and ultimately better generative models.  

\subsection{Kernel Density Estimation} 

We propose to use a non-parametric method to estimate a density in the space of latent code vectors. We use kernel density estimation (KDE) as it is commonly used in the literature (see e.g.~\cite{nowozin2016f}). We can then use the resulting density to drive the generation of new examples, controlling the generalization of the model vs.~pure memorization by the bandwidth of the kernel. More concretely, given a translation-invariant kernel $k$, bandwidth $h$, and a sample of codes $\{\z_i\}_{i=1}^n$, we can approximate the probability density at point $\z$ as
$\hat p(\z) = \frac{1}{nh}\sum_{i=1}^n k \left(\frac{\z-\z_i}{h}\right)$.
In this work, we chose a kernel density estimator with an RBF kernel. Note that sampling from this distribution is straightforward.
This approach can also be combined with mixed sampling from a broader background distribution in order to prevent overfitting to the training data.

\section{Improved GAN Training}
\label{sec:improved_GAN}

What we have done so far is to introduce a technique to estimate a flexible prior over the latent codes given by a generator $G_\phi$. This technique revealed a clear structure in the latent space as suggested by the results shown in Figure~\ref{fig:structure:mnist}.
Yet this information is typically ignored by the standard training procedure for GANs, which we review below.

\paragraph{Standard GAN training.}
The main idea behind GANs is to pit two networks - a generator and a discriminator - against each other. The generator $G_\phi$ takes as input a random noise vector $\z$ drawn from $\mathcal P_\z = \mathcal N(\mathbf 0, \mathbf I)$, and outputs a sample $\hat{\x} = G_\phi(\z)$. The other network, the discriminator $D$ attempts to differentiate the generator sample $\hat{\x}$ from samples drawn from the true distribution. Formally, the two networks play the following minimax game:
\begin{align}
\min_{G_\phi} \max_D \; & \mathbb{E}_{\x\sim \mathcal P_\x}\left[\log D(\x)\right]
+ \mathbb{E}_{\z\sim \mathcal P_\z}\left[\log (1-D(G_\phi(\z)))\right].
\label{eq:GAN_objective}
\end{align}

In~\cite{Goodfellow:2014td}, the authors suggested minimizing Eq.~\ref{eq:GAN_objective} by alternatingly optimizing the parameters of $G_\phi$ and $D$ using minibatch stochastic gradient descent (SGD) training. Under the assumptions that $G_\phi$ and $D$ have enough capacity, and that at each step of SGD, the discriminator is allowed to reach its optimum then the generator will match the data distribution. However these assumptions often do not hold in practice and consequently many problems arise when training GANs. A common problem encountered in practice is the imbalance  between the discriminator and the generator. To overcome this difficulty, one can either make the discriminator weaker or make the generator stronger. Most existing work focusing on
weakening the discriminator~\cite{Salimans:2016wg, sonderby2016amortised} has faced difficulties to find the right level by which one should "dumb-down" the discriminator. Here we pursue the direction of making the generator stronger by being able to stay closer to the true data distribution. 

\paragraph{GAN training with a flexible prior.}
We now describe a novel training procedure for GANs that uses the reversal technique presented in Section~\ref{sec:generator_reversal} to continually reconstruct latent representations of data samples. By doing this, we obtain the real data's empirical latent distribution. We will later demonstrate that this yields significant speeds up at training time. The method we propose is detailed in Algorithm~\ref{alg:improved_GAN_training}. It requires constructing and continuously updating a prior $\mathcal P_\z$ durng GAN training. Although this might seem like a costly procedure, we will demonstrate that relatively few gradient steps in the generator reversal process are required at each step of the GAN training loop.

\begin{algorithm}
    \caption{Improved GAN training}\label{alg:improved_GAN_training}
    \begin{algorithmic}[1]
        \FOR {$i = 1:T$}
        \STATE Estimate $\mathcal P_\z$ using the generator reversal process
        \STATE Sample minibatch of noise samples from $\mathcal P_\z$
		\STATE Sample minibatch of examples from $\mathcal P_\x$
        \STATE Update the discriminator and generator using SGD
        \ENDFOR
    \end{algorithmic}
\end{algorithm}

Ideally, we would like to perform enough steps such that the obtained latent representations is informative enough for the training procedure. This inevitably raises the question of what would be a good heuristic to determine the quality required from generator reversal. Perhaps the most naive way would be to use a fixed number of steps or fix a threshold on the reconstruction loss. However these two approaches might waste computation as they disregard the GAN objective presented in Eq.~\ref{eq:GAN_objective}. We instead argue for a different criterion:

\begin{claim}
The minimum level of accuracy required for the obtained latent representations to be useful is determined by the possibility to differentiate - in the latent space - generated samples from real data samples.
\label{claim:accuracy_level}
\end{claim}

Our claim closely follow the underlying principle of the minimax game where the discriminator tries to differentiate between real samples drawn from $\mathcal P_\x$ and samples generated by the generator $G_\phi$.

Following this principle we suggest using a statistical test to determine if, in the latent space, samples from the generator are drawn from the same distribution as $\mathcal P_\x$. As demonstrated by our experiments, this will naturally lead to crude reconstructions at the beginning of training, while increasing the reconstruction quality as the generator generates samples more similar to the true samples. In the following, we detail our choice of a specific statistical test.

\paragraph{Two-sample test.}

The two-sample problem, also known as homogeneity testing, is a classic problem in statistics.
Given two i.i.d.\ samples $\{x_i\}_{i=1}^m \sim p$ and $\{y_i\}_{i=1}^m \sim q$, one would like to test the hypothesis $\mathcal{H}_1: p \neq q$ (i.e. the two samples are drawn from different distributions) against the null hypothesis $\mathcal{H}_0: p = q$.
This is equivalent to testing whether $\gamma(p, q) > 0$ for a metric $\gamma$ on the space of probability measures. In this work, we choose $\gamma$ to be the \emph{Maximum Mean Discrepancy (MMD)}~\cite{Gretton:2012wt} which has been applied to the setting of unsupervised learning with GANs in~\cite{Sutherland:2016wi}.

Given two pairs of independent random variables $(\x, \x') \sim p$ and $(\y, \y') \sim q$ and a characteristic kernel \cite{Fukumizu:2007wz} $k$ with associated RHKS $\mathcal{H}$, the squared population MMD is defined as
\begin{equation}
\text{MMD}^2\left[k, p, q\right] = \mathbb{E}_{\x,\x'}\left[k(\x,\x')\right] + \mathbb{E}_{\y,\y'}\left[k(\y,\y')\right] - 2\mathbb{E}_{\x,\y}\left[k(\x,\y)\right].
\end{equation}

As shown in~\cite{Gretton:2012wt} various unbiased estimators of $\text{MMD}^2\left[k, p, q\right]$ can be computed. We will here use a quadratic estimator defined as
\begin{equation}
    \text{MMD}_u^2\left[k, p, q\right] = \frac{1}{m(m-1)}\sum_{i\neq j}^m \left[ k(x_i, x_j) + k(y_i, y_j) - k(x_i, y_j) - k(x_j, y_i) \right].
    \label{eqn:mmd}
\end{equation}

The time complexity for computing this estimator is quadratic in the number of samples $\mathcal{O}(m^2)$, but linear time $\mathcal{O}(m)$ approximations exist~\cite{Gretton:2012wt}.  In this work, we choose a Gaussian (RBF) kernel for $k$, since it can be efficiently computed and fulfils the requirements of a characteristic kernel.

\paragraph{MMD as a stopping criterion.}
We now detail how MMD is used in Algorithm~\ref{alg:improved_GAN_training} as a stopping criterion. Following claim~\ref{claim:accuracy_level}, what we would like is to check whether the distribution produced by the generator network in the latent space $\mathcal Q_\z$ is a good fit for the latent distribution of the true data $\mathcal P_\z$. As evaluating the whole distribution would be too expensive, we instead estimate the MMD distance on minibatches of samples from $\mathcal P_\z$ and $\mathcal Q_\z$ using the estimator in Eq.~\ref{eqn:mmd} after each step of generator reversal.
Once we can confidently distinguish the reconstructed latent representations of the true data from those of the generated data, we terminate the reversal procedure.

Figure~\ref{fig:mmd} shows the maximum mean discrepancy between the obtained latent representations of a data mini-batch and the obtained latent representations of a mini-batch of generated samples as a function of reconstruction steps taken.
The increase in MMD matches our expectations: with better reconstructions, we can more confidently distinguish the data's latent representations from those of the generated samples.
Further, the graphs show that the MMD continues to increase even after a large number of steps.

\begin{figure}[t!]
    \begin{minipage}{0.45\textwidth}
\begin{center}
\centerline{
    \includegraphics[width=\columnwidth/2]{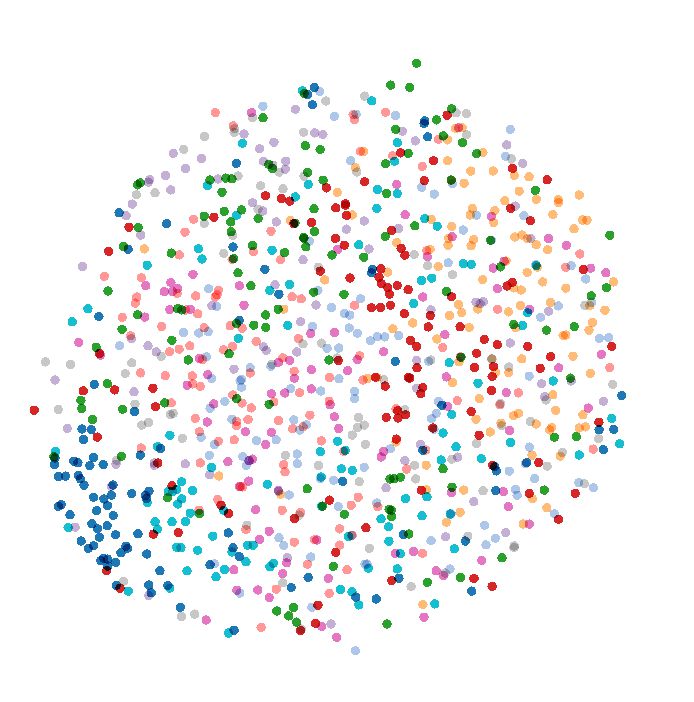}
    \includegraphics[width=\columnwidth/2]{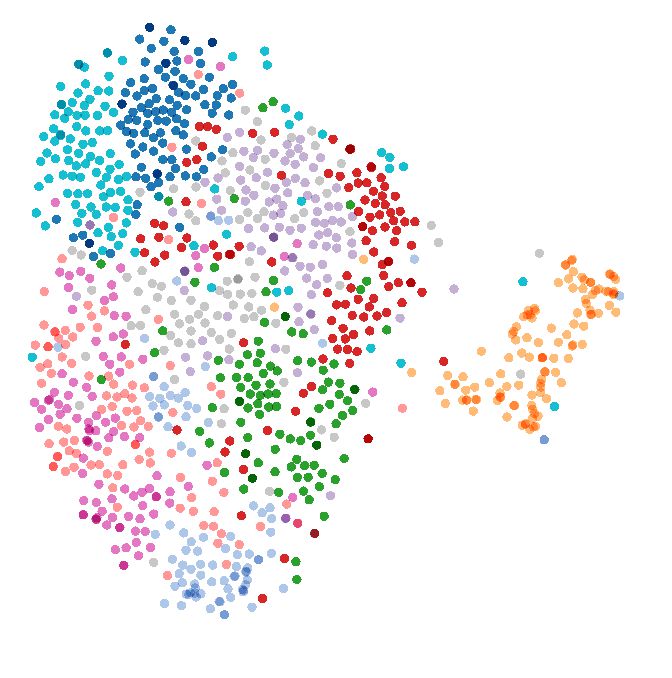}
}
\caption{\small{\it Generator reversal on a sample of 1024 MNIST digits. Projections of data points with 
an untrained (left) and a fully trained GAN (right). Colors represent the respective class labels. The ratios of between-cluster distances to within-cluster distances are 0.1 (left) and 1.9 (right). }}
\label{fig:structure:mnist}
\end{center}
    \end{minipage}
    \hspace*{1.2cm}
    \begin{minipage}{0.45\textwidth}
\begin{center}
\centerline{
    \includegraphics[width=0.9\columnwidth]{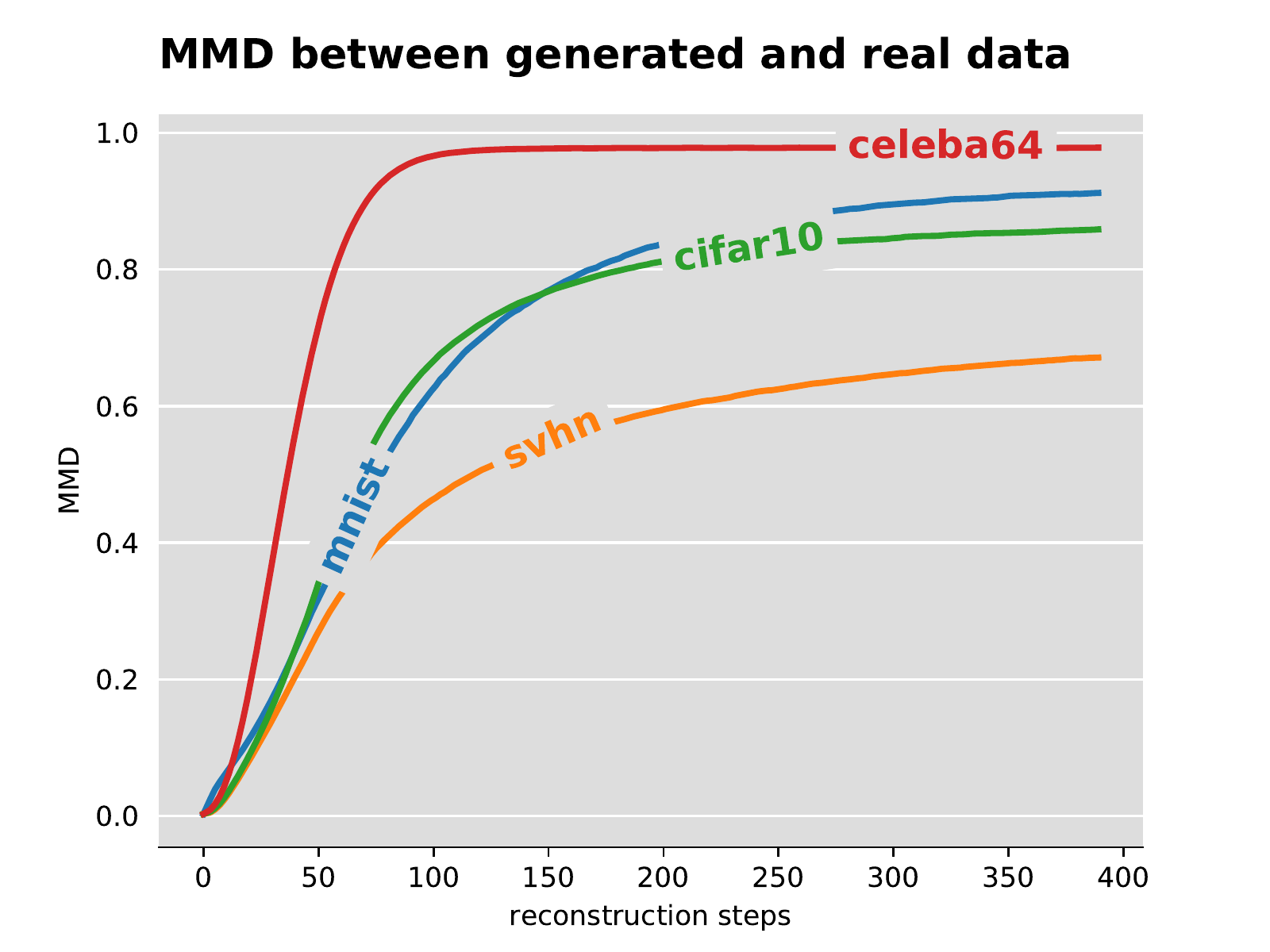}
}
\caption{\small{\it MMD in the latent space between a mini-batch of real data samples and a mini-batch of samples produced by the generator.}}
\label{fig:mmd}
\end{center}
    \end{minipage}
\end{figure} 

\paragraph{Experimental results.}
Our experimental setup closely follows popular setups in GAN research and is detailed in the Appendix.

Figure~\ref{fig:kde} highlights two important results from our experiments:
In the top row we plot the inception score~\cite{Salimans:2016wg} on a test set over the course of training.
As can be seen, our method reaches higher scores more quickly and retains this advantage until the end of training.

In the bottom row we plot the number of reconstruction steps that are taken for each training step. Recall that this is determined by the MMD two-sample test.
As expected, the number of steps needed to produce useful reconstructions increases during training, if only slightly.
More details and results of the training procedure can be found in the Appendix.

\paragraph{Visual results.}
Figures~\ref{fig:samples_progress} and~\ref{fig:samples_full_model} compare samples generated from a standard GAN and a KDE GAN as training progresses and at the end of training. The images in Figure~\ref{fig:samples_progress} clearly demonstrate the faster progress of KDE GAN which generates visually better samples in the early stage of the optimization procedure. The samples generated by the fully trained model in Figure~\ref{fig:samples_full_model} also appear to be of superior quality.

\paragraph{Manifold traveral.}
Given two seeding images from the dataset, we linearly interpolate between their latent representations, which we obtain by generator reversal. The resulting images in Figure~\ref{fig:traversal} show that we learn a smooth transition function between the data space and the latent space. This indicates that the better performance of KDE GAN is not at the detriment of overfitting to the training data, since the model has a notion of semantically neighbouring images. This is also confirmed by our neighborhood exploration experiments, given in the Appendix.

\begin{figure*}[t!]
\begin{center}
\centerline{
    \includegraphics[width=0.25\columnwidth]{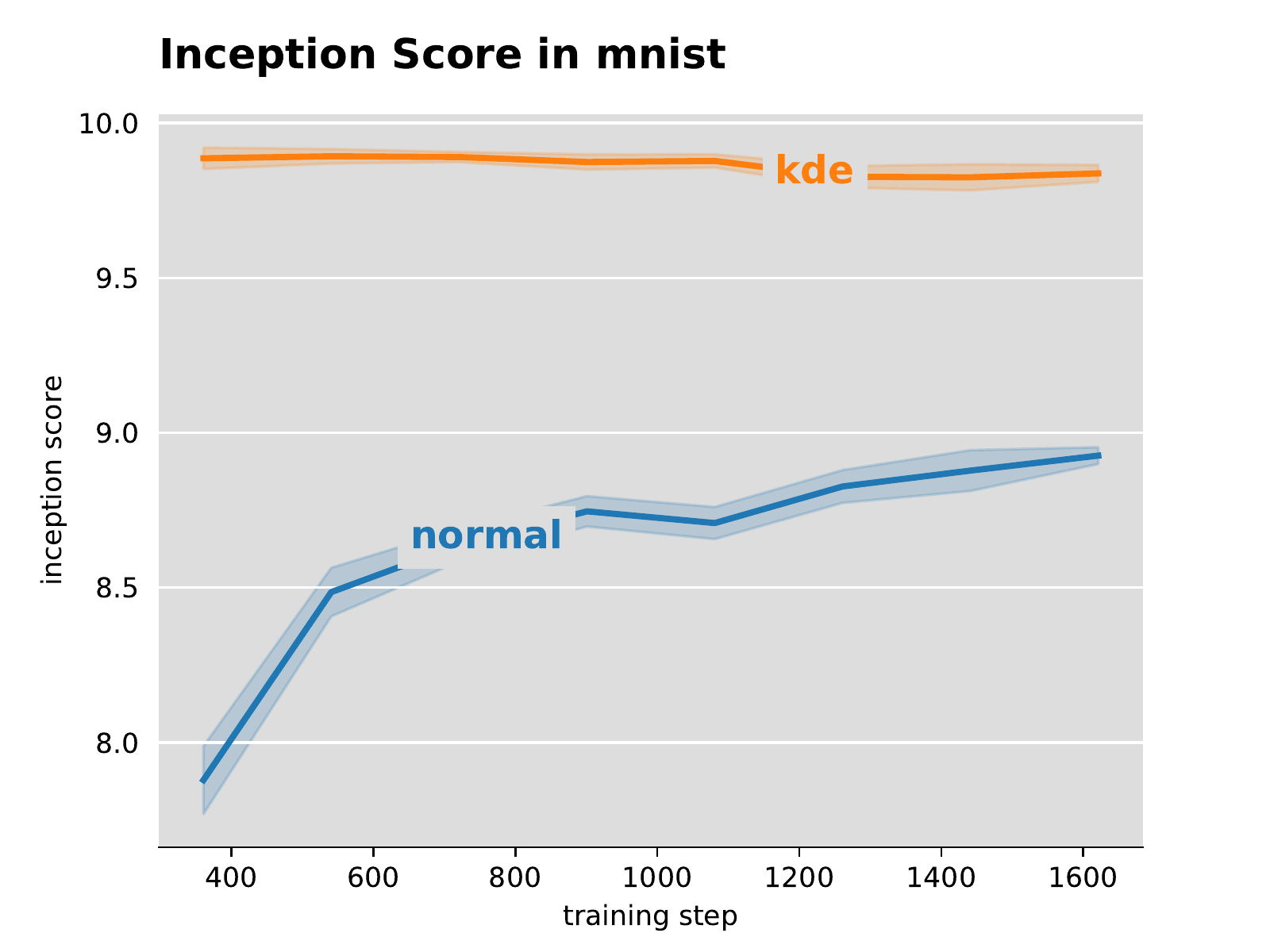}
    \includegraphics[width=0.25\columnwidth]{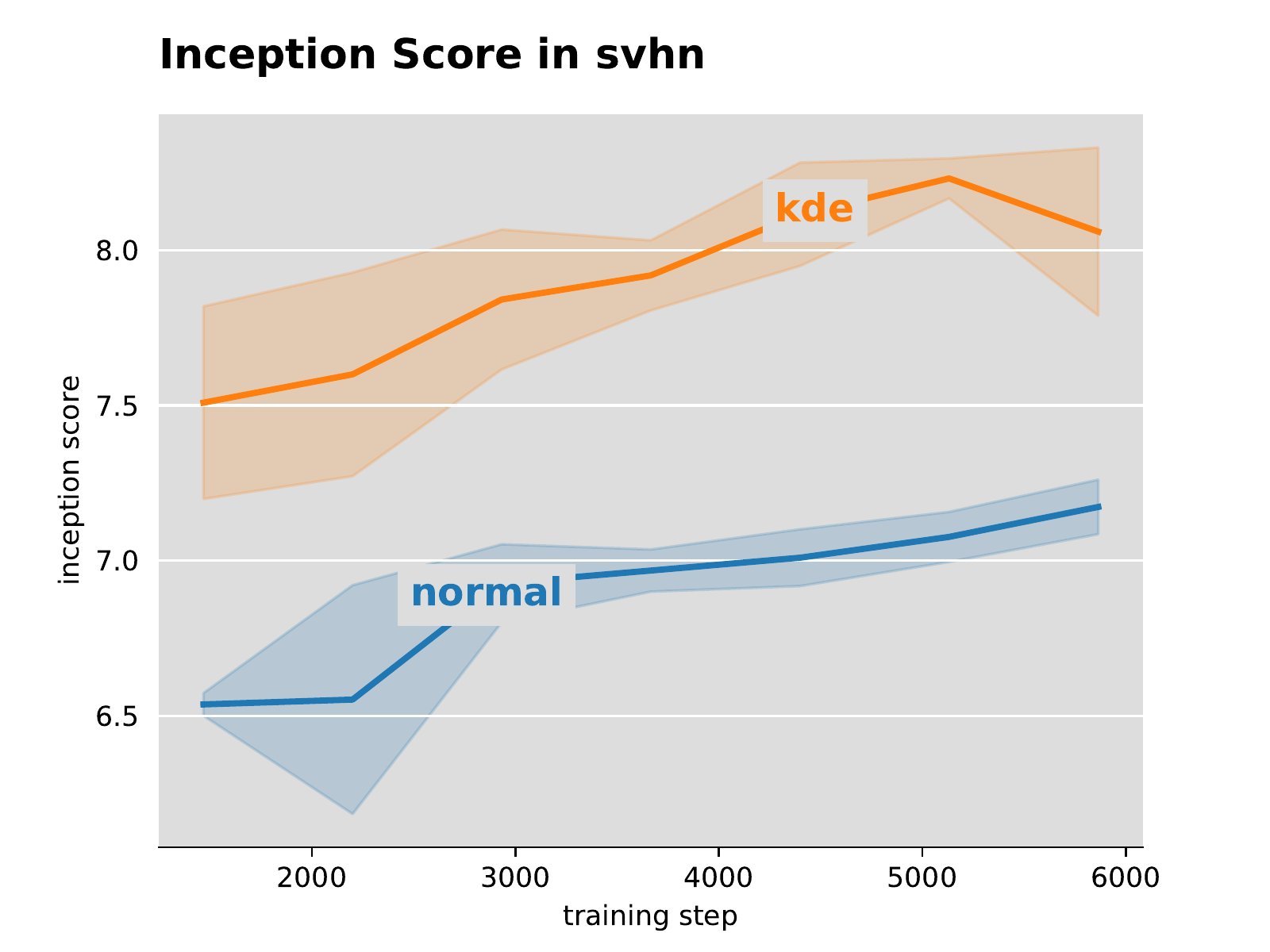}
    \includegraphics[width=0.25\columnwidth]{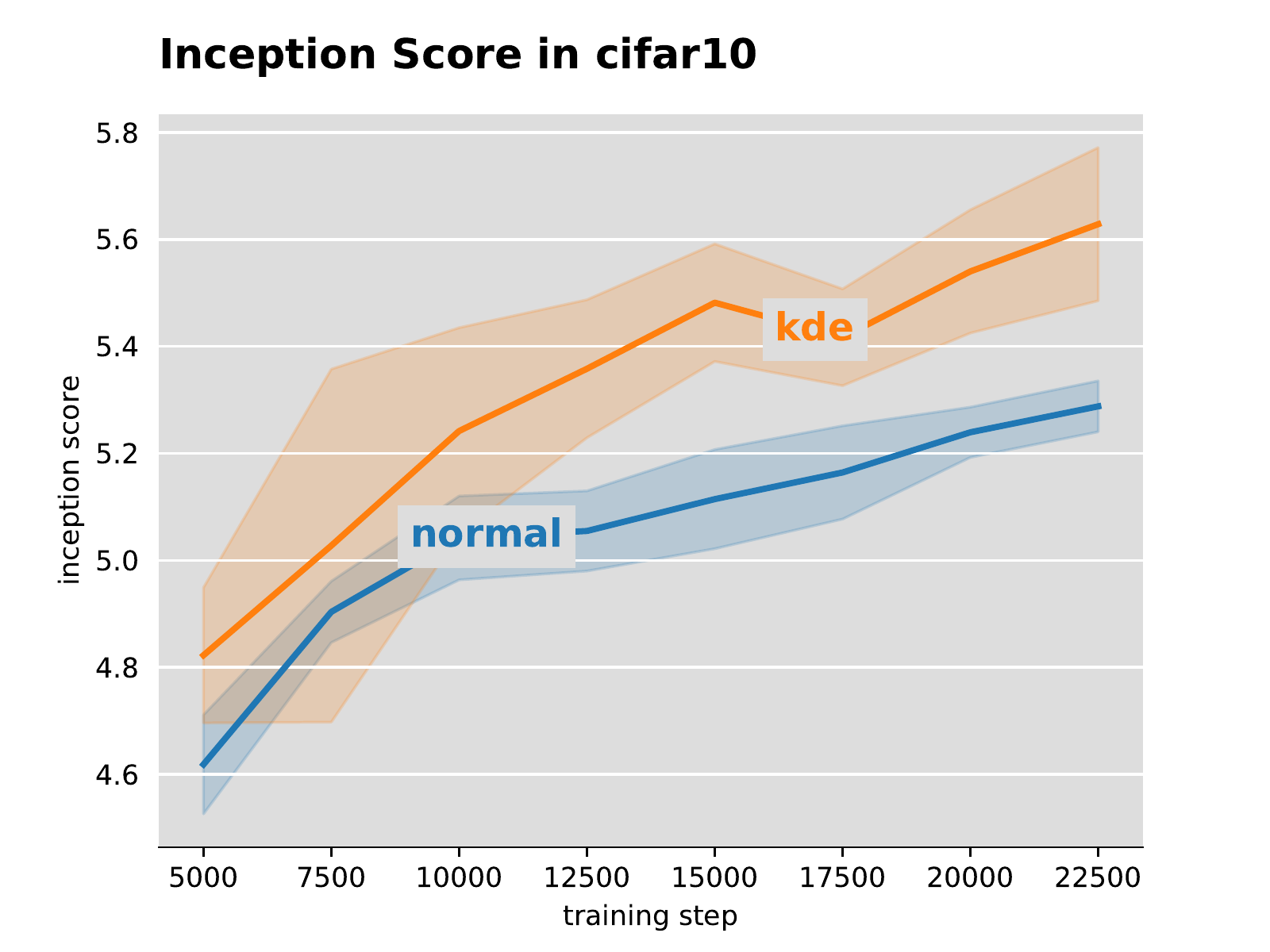}
    \includegraphics[width=0.25\columnwidth]{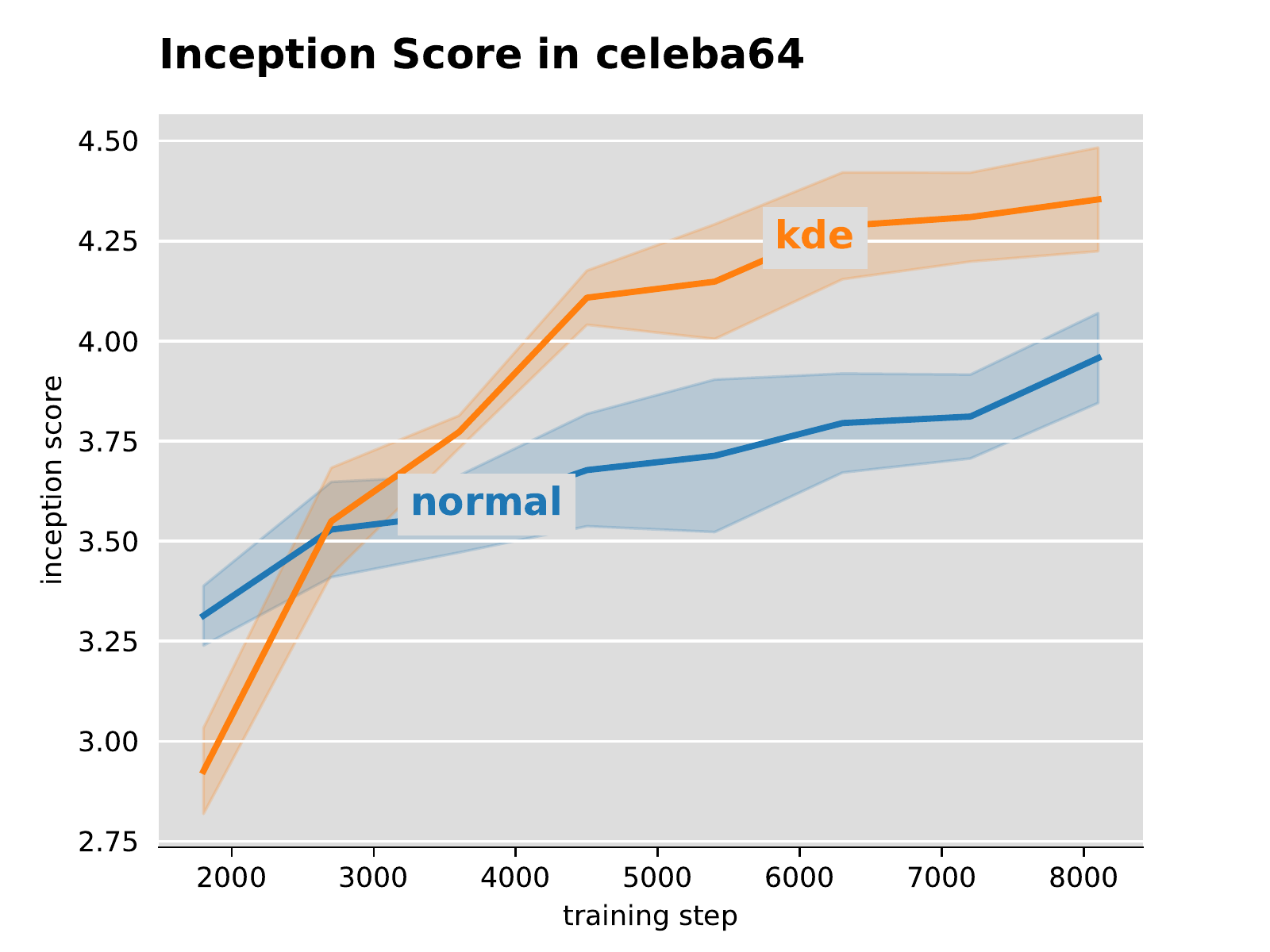}
}
\centerline{
    \includegraphics[width=0.25\columnwidth]{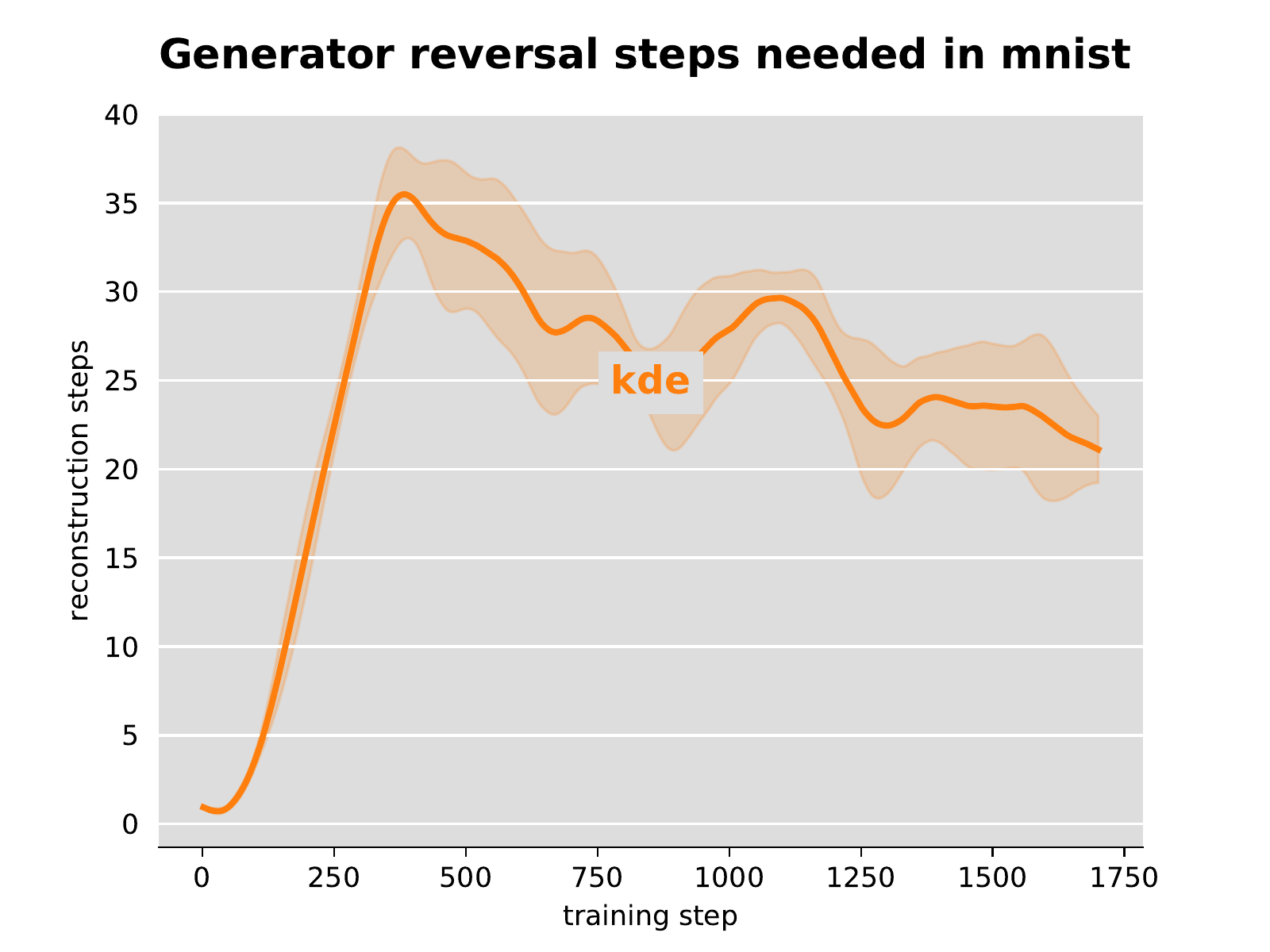}
    \includegraphics[width=0.25\columnwidth]{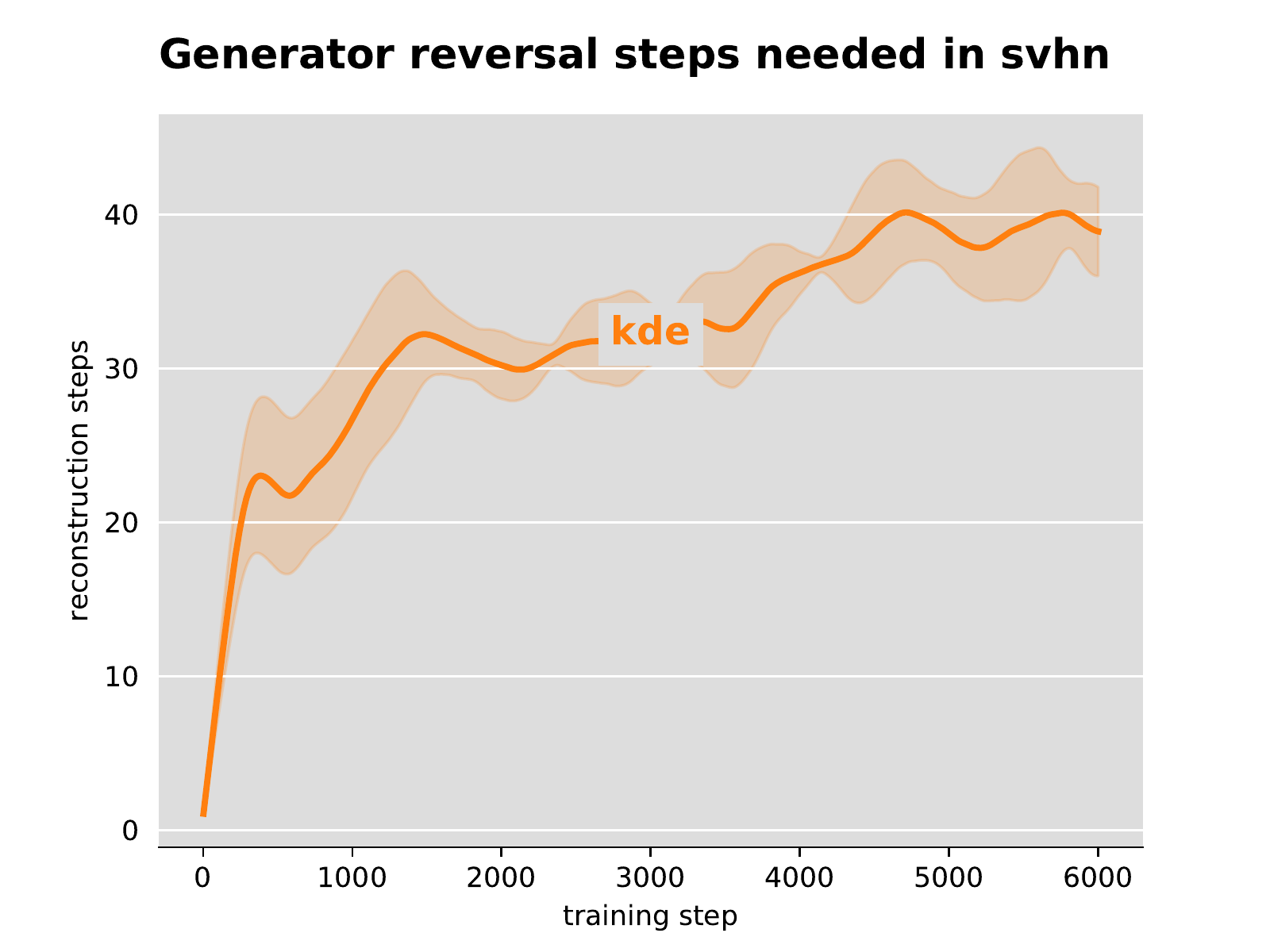}
    \includegraphics[width=0.25\columnwidth]{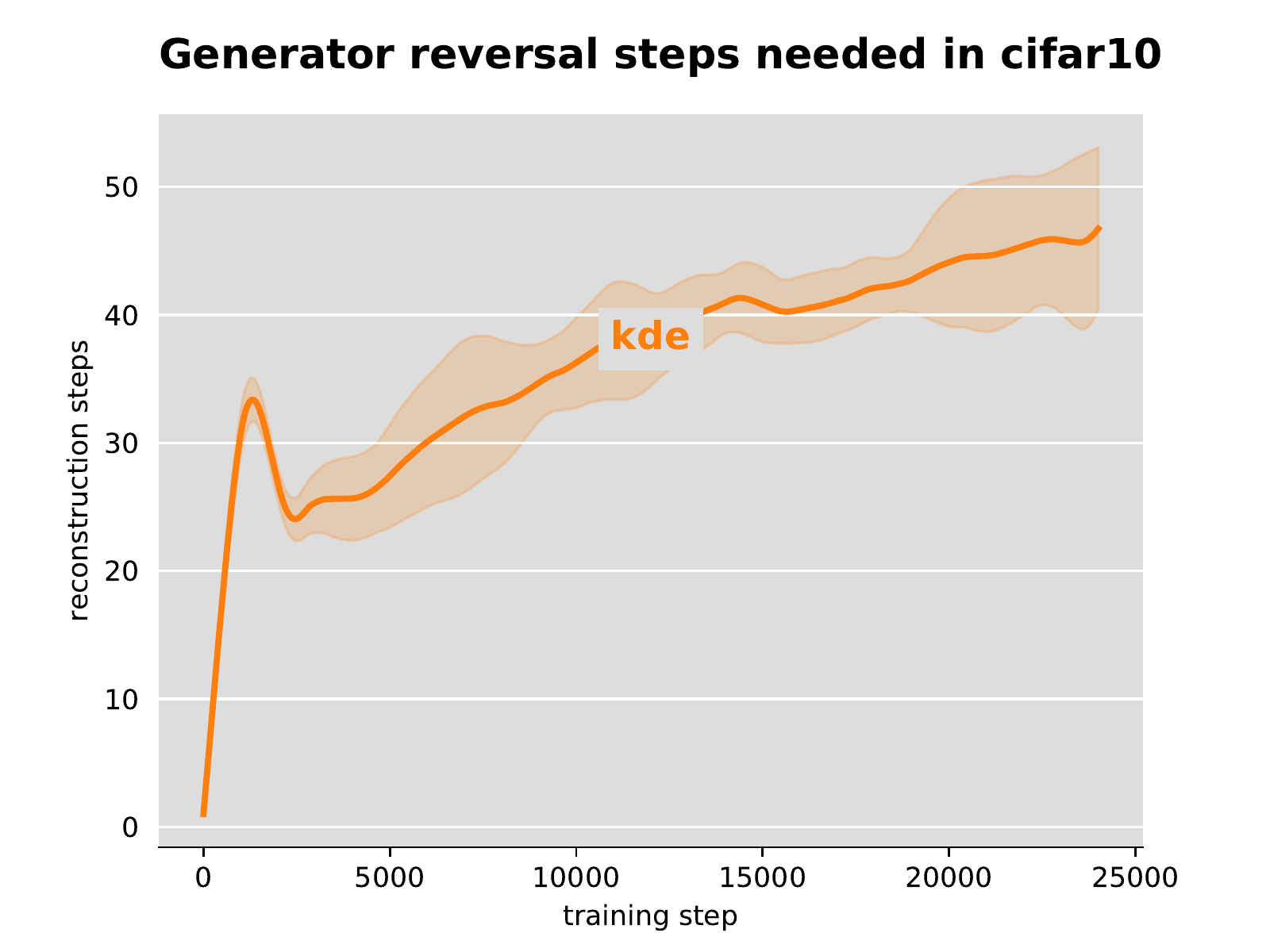}
    \includegraphics[width=0.25\columnwidth]{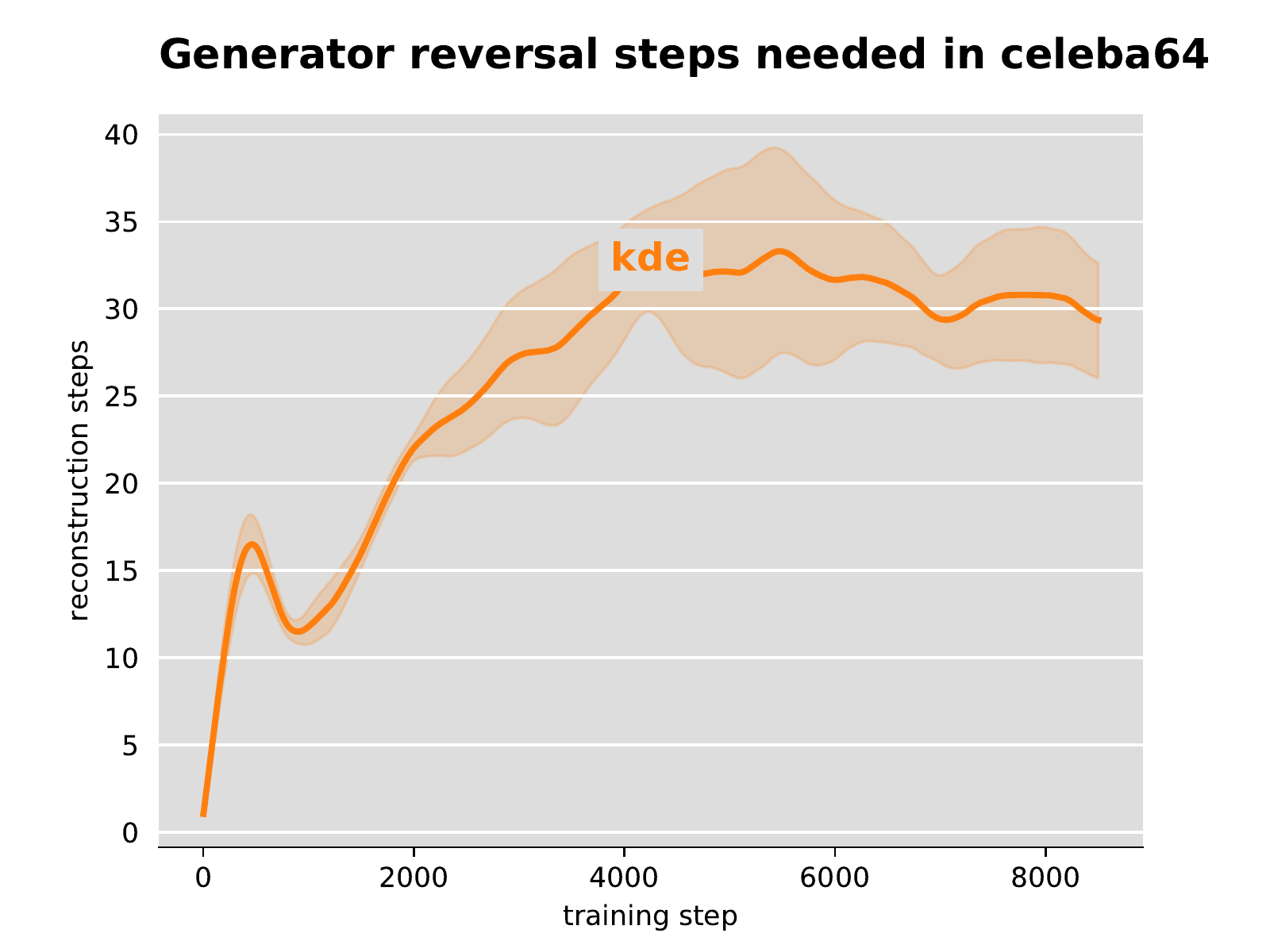}
}

\caption{\small{\it 
    Top row: Test set inception score over the course of the training procedure. Bottom row: The number of reconstruction steps taken by the generator reversal process in the inner loop until the MMD threshold is reached. Shaded areas denote one standard deviation around the mean, averaged over five training runs.}}
\label{fig:kde}
\end{center}
\end{figure*} 

\begin{center}
\begin{figure}[t!]
\foreach \ds in \datasetsa {
\begin{subfigure}[p]{0.49\textwidth}
    \includegraphics[width=1.0\columnwidth]{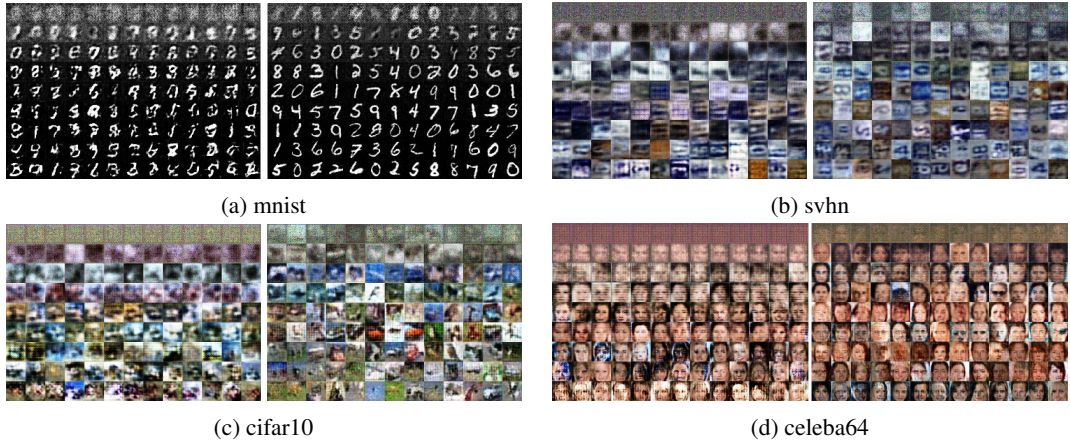}
    \caption{\ds}
    \label{fig:prog:\ds}
\end{subfigure}\quad
}
\caption{\small{{\it Image samples from the generative model at the beginning of training. Vanilla DCGAN samples on the left, our KDE GAN on the right. Each row is sampled after 50 steps of training, starting at step 100.}}}
\label{fig:samples_progress}
\end{figure} 
\end{center}

\begin{center}
\begin{figure}[t!]
\foreach \ds in \datasetsa {
\begin{subfigure}[p]{0.22\textwidth}
    \includegraphics[width=1.1\columnwidth]{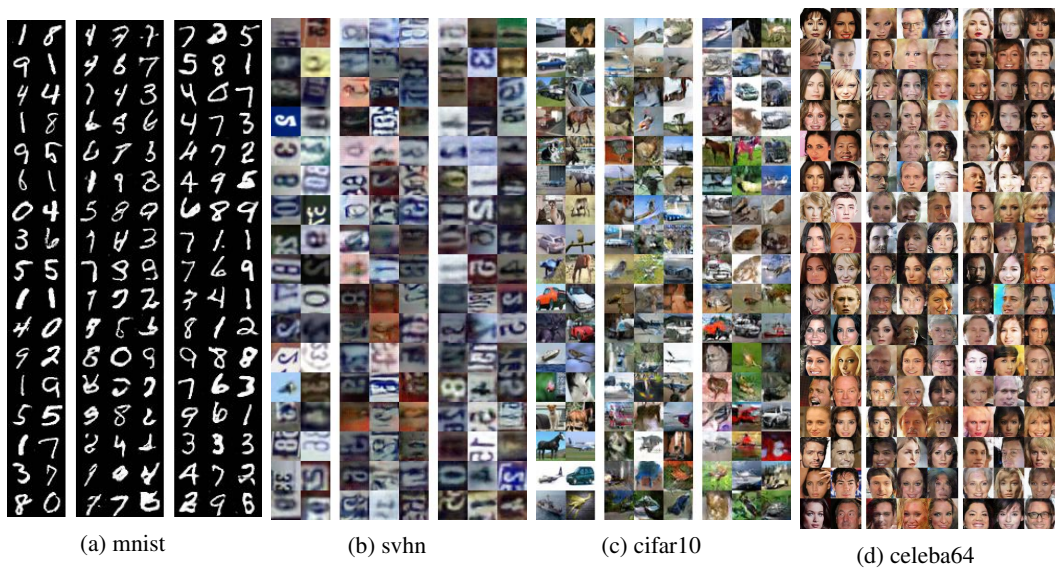}
    \caption{\ds}
    \label{fig:final:\ds}
\end{subfigure}\quad
}
    \caption{\small{{\it Dataset images (left) and samples from fully trained vanilla DCGAN (middle) and KDE GAN (right).}}}
\label{fig:samples_full_model}
\end{figure}
\end{center}

\begin{figure*}[t!]
\begin{center}
\centerline{
    \includegraphics[width=1.05\columnwidth]{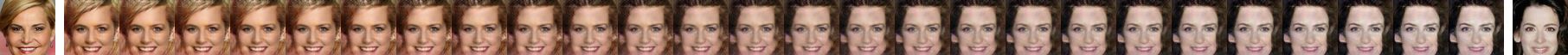}
}
    \vspace*{0.05cm}
\centerline{
    \includegraphics[width=1.05\columnwidth]{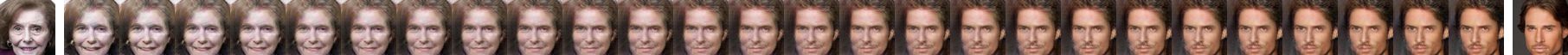}
}
    \vspace*{0.05cm}
\centerline{
    \includegraphics[width=1.05\columnwidth]{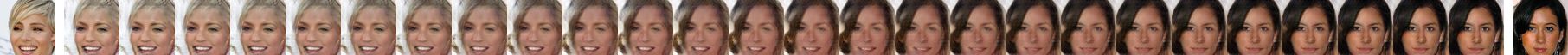}
}
    \vspace*{0.05cm}
\centerline{
    \includegraphics[width=1.05\columnwidth]{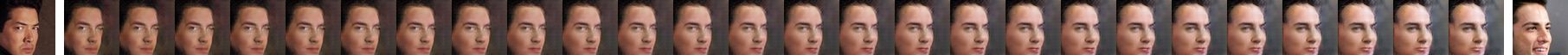}
}
    \vspace*{0.05cm}
\centerline{
    \includegraphics[width=1.05\columnwidth]{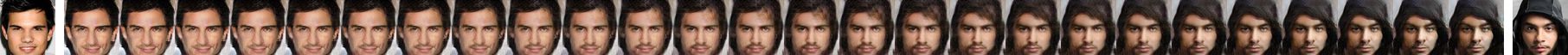}
}
    \caption{\small{{\it Manifold traversal with KDE GAN. We linearly interpolate between the latent codes of two given seeding images (far left and right). 
    }}}
\label{fig:traversal}
\end{center}
\end{figure*} 


\section{Related Work}

\paragraph{VAEs vs GANs.}
Another example of a popular technique to train a neural generative model is the Variational Auto-encoder approach proposed by~\cite{Kingma:2013tz, Rezende:2014vm} which explicitly models the data distribution in the latent space and learn its parameters using an encoder neural network to minimize the reconstruction error. A common criticism against VAEs is that they tend to generate blurry samples, a problem which is typically alleviated by GANs.

However, GANs lack an efficient inference mechanism. This problem was addressed by some recent work, e.g.~\cite{Che:2016wq, Dumoulin:2016td, donahue2016adversarial} which focus on training a separate \emph{encoder} network in order to map a sample from the data space to its latent representation. Our goal is however different as our procedure is used to estimate a more flexible prior over the latent space. The importance of using an appropriate prior for GAN models has also been discussed in~\cite{han2016alternating} which suggested to infer the continuous latent factors in order to maximize the data log-likelihood. However this approach still makes use of a simple fixed prior distribution over the latent factors and does not use the inferred latent factors to construct a prior as suggested by our approach.

As discussed throughout this paper, standard GAN training requires to strike a good balance between the two networks. The work of~\cite{Metz:2016wg} introduced a method to stabilize training by defining the generator objective with respect to an unrolled optimization of the discriminator. They show that a gradient update procedure can be effectively included into the back-propagation step of another gradient update procedure. Although our approach is fundamentally different, it does also stabilize training and can also be implemented end-to-end as part of the back-propagation step.

\paragraph{MMD.} The Maximum Mean Discrepancy measure~\cite{Gretton:2012wt} is a particular instance of an \emph{integral probability metric}~\cite{muller1997integral}, a class of metrics on probability measures that includes a wide variety of known divergences, such as the Kolmogorov Distance, the Total Variation Distance and the Wasserstein Distance. For an excellent discussion of the relation of MMD to other integral probability metrics as well as other commonly used divergence measures, such as the Kullback-Leibler Divergence, we refer the reader to~\cite{Sriperumbudur:2010wp}. 
The connection between MMD and GANs have also been explored in the literature. Both \cite{dziugaite2015training} and \cite{li2015generative} proposed an approximation to adversarial learning that replaces the discriminator with the MMD criterion in the data space.

\section{Conclusion}

We presented a novel approach to estimate a flexible prior over the latent codes
given by a generator $G_\phi$. This is achieved through a reversal technique that continually reconstruct latent representations of data samples and use these reconstructions to construct a prior over the latent codes. We empirically demonstrated that this reversal technique yields several benefits including: more powerful generative models, better modeling of latent structure and explicit control of the degree of generalization.


\bibliography{nips_2017}
\bibliographystyle{plain}


\newpage
\appendix
\label{section:app}

\section{Detailed Experiment Setup.}

Our experimental setup closely follows popular setups in GAN research in order to facilitate reproducibility and enable qualitative comparisons of results.
Our network architectures are as follows:

The generator samples from a latent space of dimension 200, which is fed through a padded deconvolution to form an initial intermediate representation of $4\times 4\times 512$, which is then fed through three layers of deconvolutions with 256, 128 and 64 filters, followed by a last deconvolution to get to the desired output size and channels.
For the CelebA dataset, which we cropped and rescaled to image size $64\times 64$, we employ an additional deconvolutional layer with 512 filters between the first and second layer.

The discriminator consists of three layers of convolutional layers with 256, 128 and 64 filters, followed by a fully connected layer and a sigmoid classifier.
Again, for the CelebA dataset, the discriminator is augmented by an additional convolutional layer using 512 filters.

Both the generator and the discriminator use $4\times 4$ filters with a stride of $2$ in order to up- and downscale the representations, respectively.
The generator employs ReLU non-linearities, except for the last layer, which uses hyperbolic tangent.
The discriminator uses Leaky ReLU non-linearities with a leak of $0.2$, which is standard in the GAN literature.

We use RMSProp\cite{tieleman2012lecture} with a step size of $0.0003$ and mini-batches of size 100 for optimization for both the generator and discriminator.
Both networks are updated once per iteration.

For the generator reversal process, we use a learning rate of $1.0$ and an MMD threshold of $0.0001$ with a kernel bandwidth of $0.0001$.
The initial noise vectors are sampled from a normal distribution with $\sigma = 0.0001$.

We train until we can no longer see any significant qualitative imrovement in the generated images or any quantitative improvement in the inception score.
This amounts to 3 epochs on MNIST, 10 epochs on SVHN, 50 epochs on CIFAR10 and 5 epochs on CelebA.

We crop the images of the CelebA dataset to a size of $118\times 118$ pixels, after which we resize them to $64\times 64$ pixels.
Our images from MNIST, SVHN and CIFAR10 retain their original sizes of $28\times 28$, $32\times 32$ and $32\times 32$ pixels, respectively.

As for the effect of generator reversal on wall clock time, in our TensorFlow implementation, an iteration performing 5, 20 or 50 generator reversal steps takes about 1.9, 4.9 or 9.8 times as long as an iteration without generator reversal, but there is still a lot of room for optimization.

\section{Additional Training Metrics}

\begin{figure*}[h!]
\begin{center}
\centerline{
    \includegraphics[width=0.25\columnwidth]{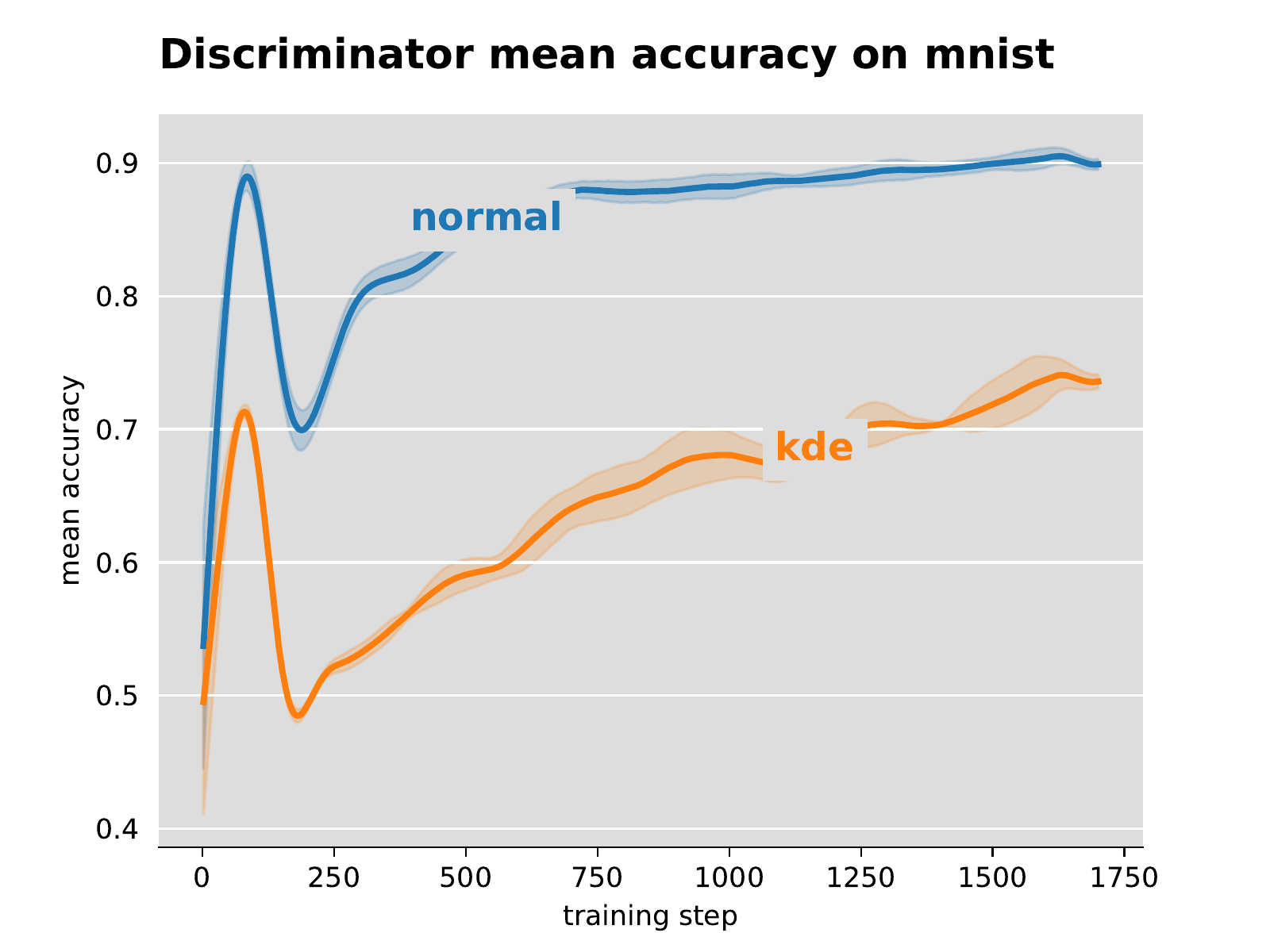}
    \includegraphics[width=0.25\columnwidth]{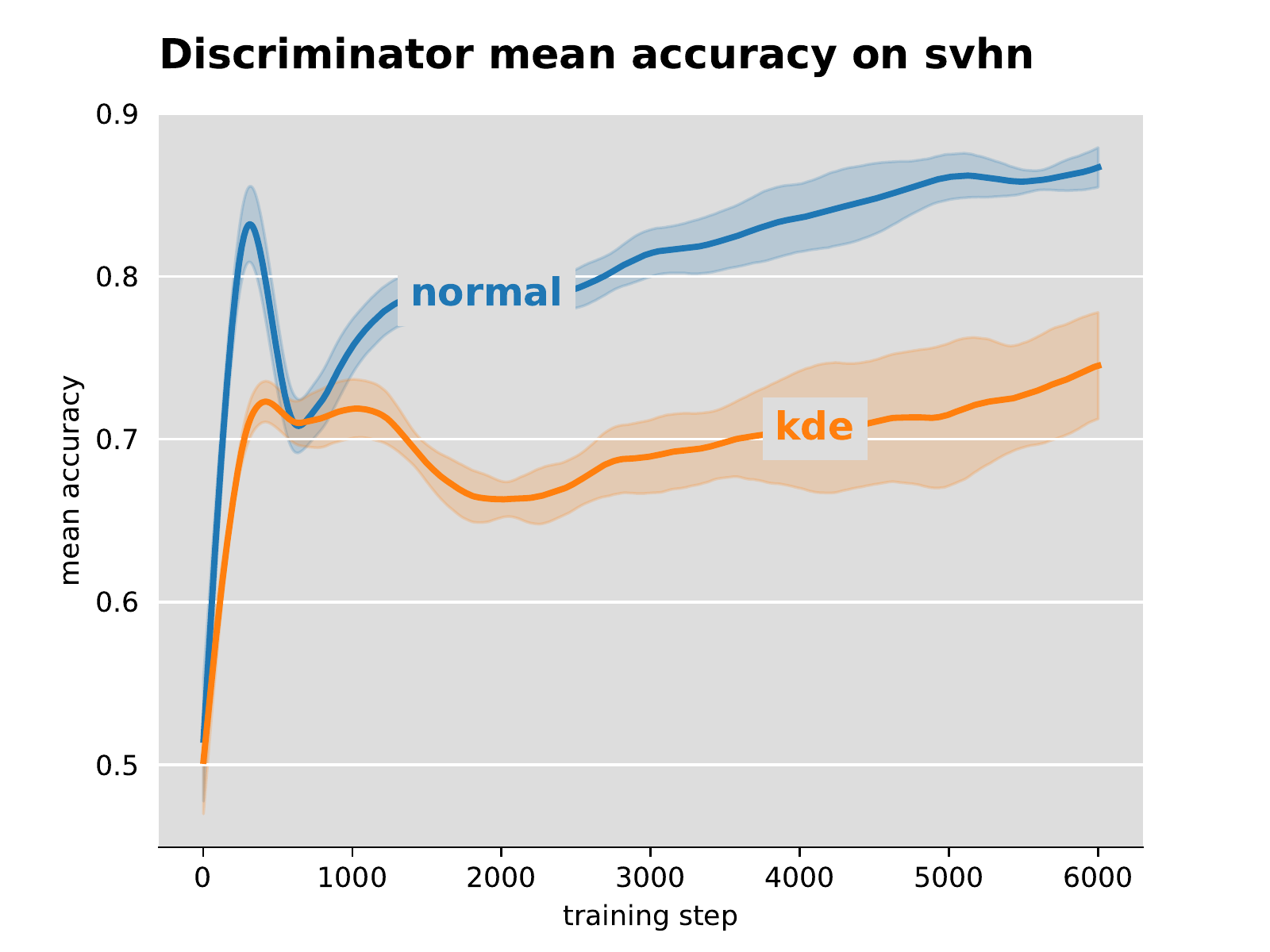}
    \includegraphics[width=0.25\columnwidth]{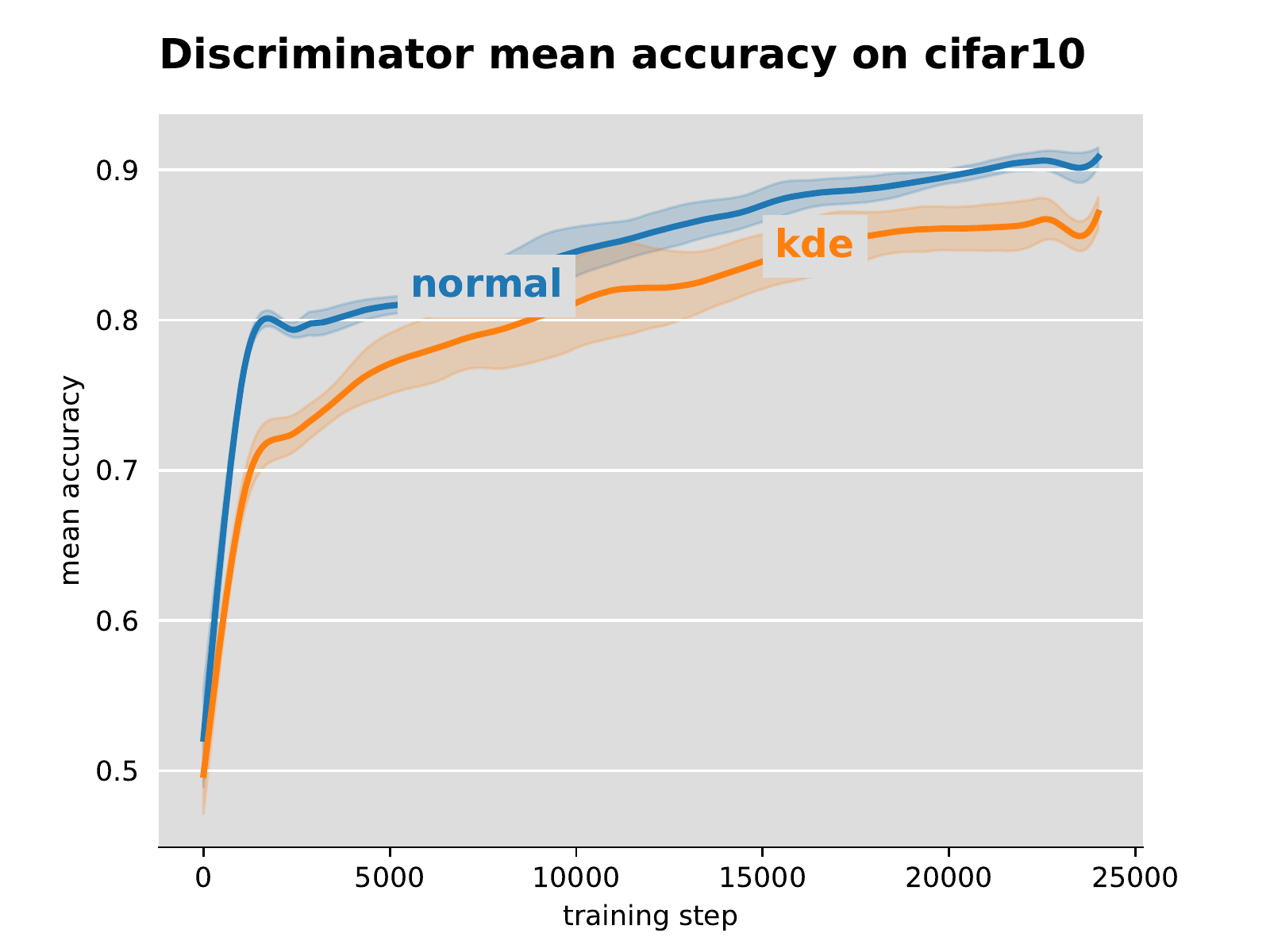}
    \includegraphics[width=0.25\columnwidth]{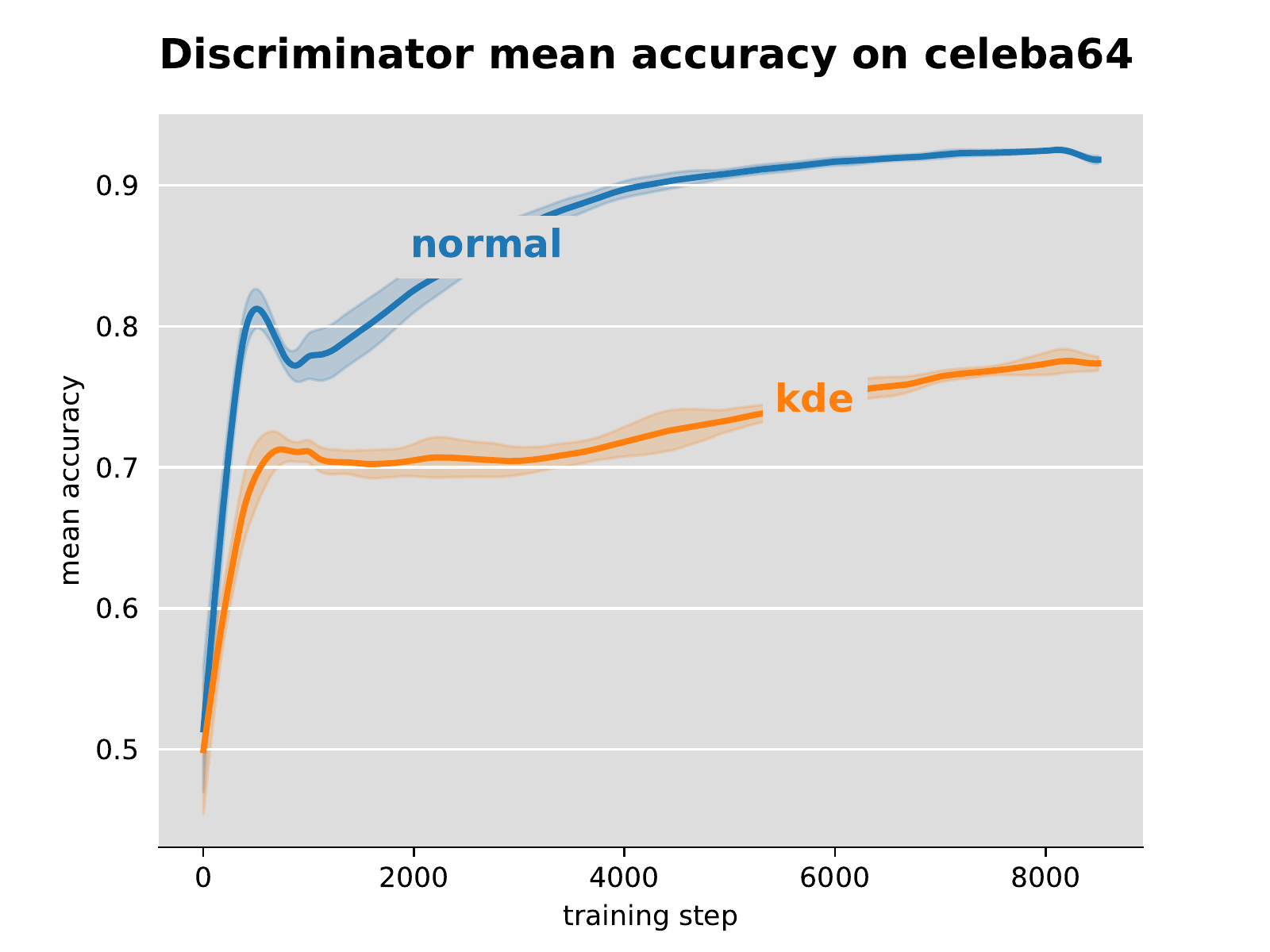}
}
\centerline{
    \includegraphics[width=0.25\columnwidth]{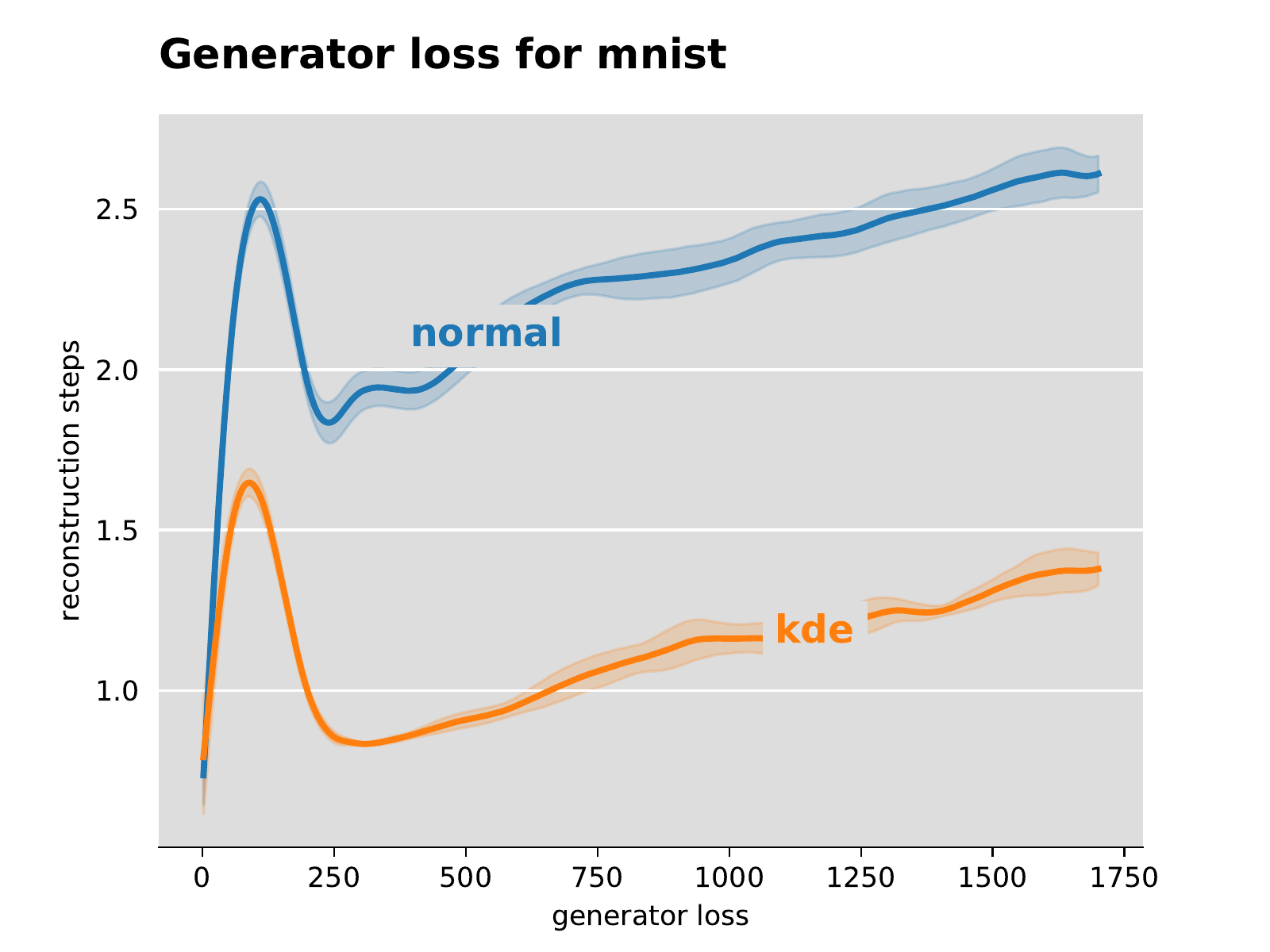}
    \includegraphics[width=0.25\columnwidth]{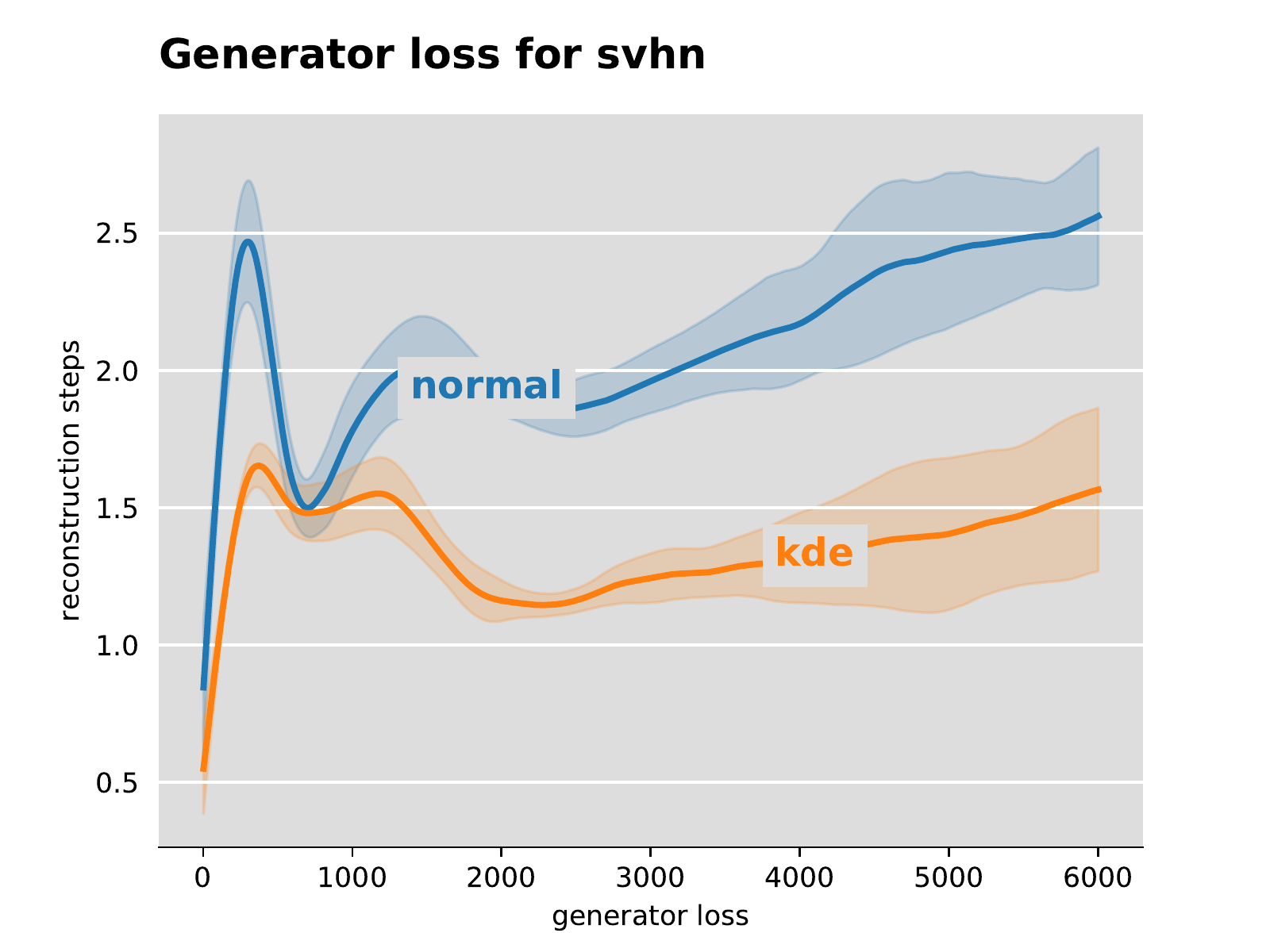}
    \includegraphics[width=0.25\columnwidth]{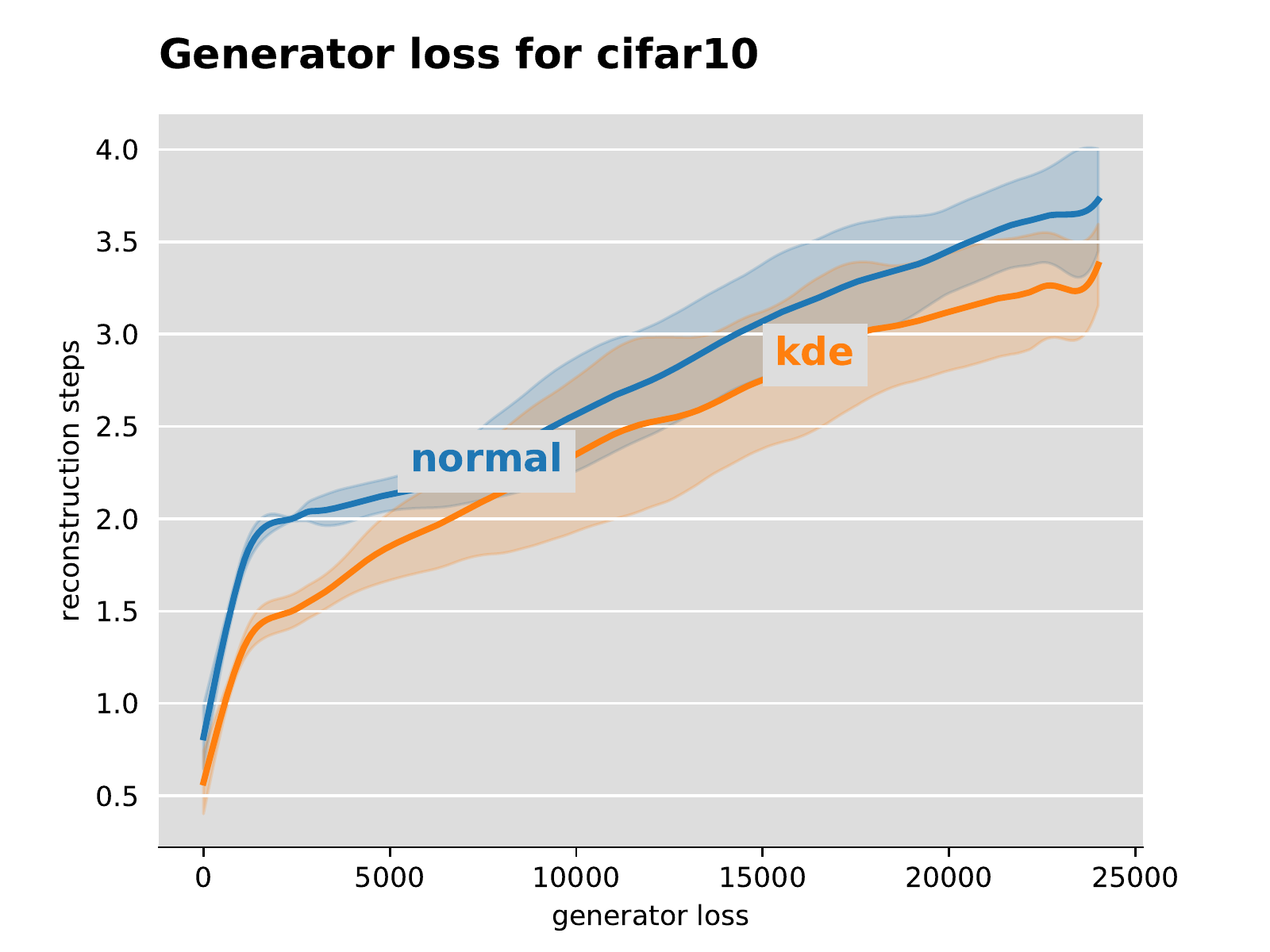}
    \includegraphics[width=0.25\columnwidth]{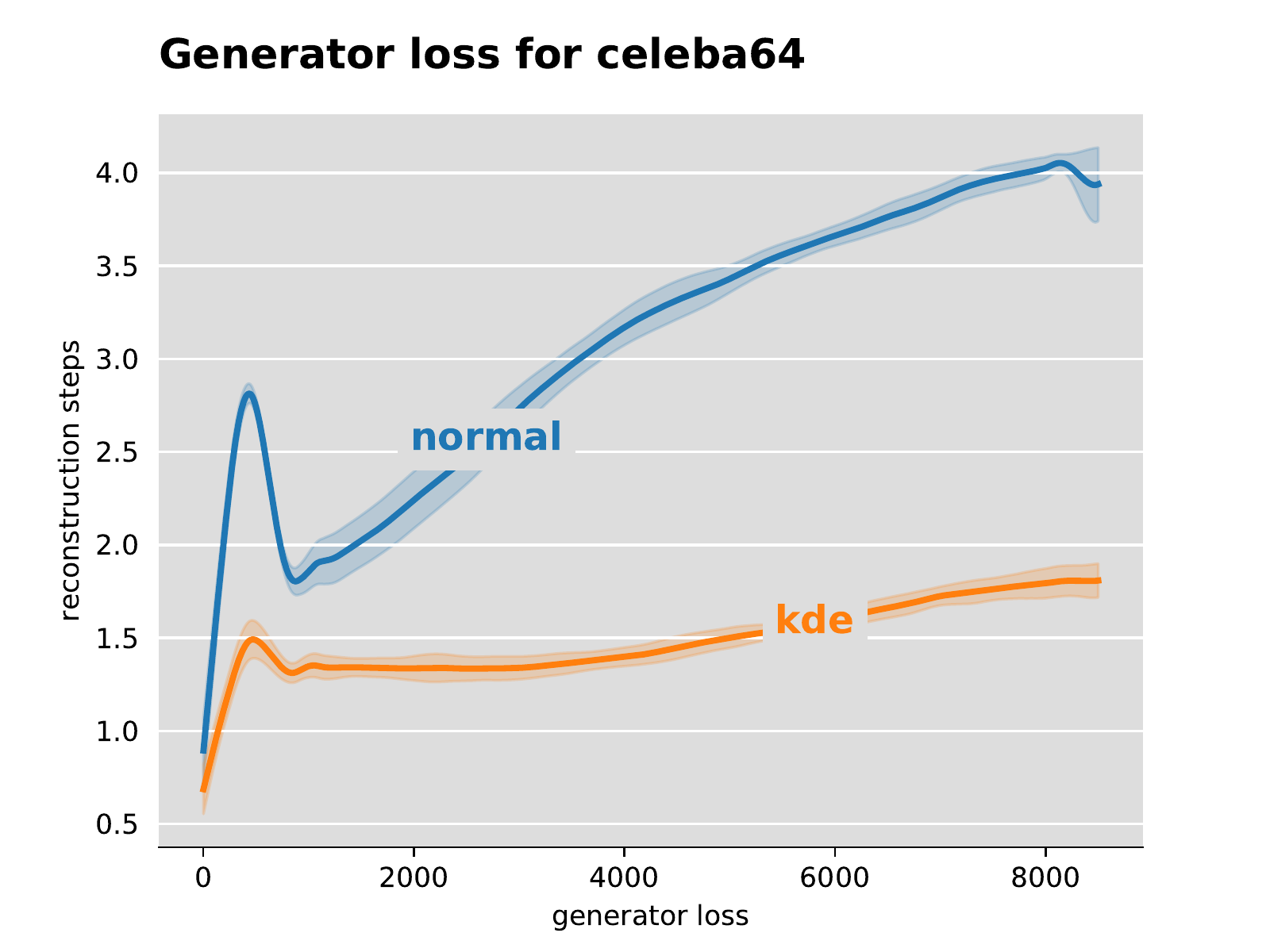}
}

\caption{\small{\it 
    Top row: Discriminator mean accuracy over the course of the training procedure. Bottom row: Training loss of the generator as training progresses. Shaded areas denote one standard deviation around the mean over five repetitions.}}
\label{fig:tmetrics}
\end{center}
\end{figure*} 

Figure~\ref{fig:tmetrics} shows the training loss of the generator as well as the mean accuracy of the discriminator over the course of training.
Both metrics indicate that our training procedure improves the generator, such that it achieves a better loss as well as manages to fool the discriminator more easily.

\section{Additional Evaluation Metrics}

\paragraph{Nearest-neighbor test.} We evaluate the generalization properties of our model using the common procedure of examining nearest-neighbors in the data space~\cite{Goodfellow:2014td, Metz:2016wg}. Specifically, we take samples from the training distribution and compute the nearest-neighbor images in pixel space from the training data. We then report the average pairwise distance between each sample and its neighbors. The results shown in Table~\ref{tbl:nndist} show that the distance is the same as the standard GAN approach which indicates that our model has similar generalization properties and thus does not simply overfit to the training data.

\begin{table}[h!]
    \centering
\begin{tabular}{l | r r}
    & normal & kde \\
    \hline
    MNIST & $4.56 \pm 0.34$ & $4.26 \pm 0.37$ \\
    SVHN & $9.93 \pm 0.12$ & $9.90 \pm 0.13$ \\
    CIFAR10 & $10.06 \pm 0.40$ & $10.11 \pm 0.12$ \\
\end{tabular}
    \vspace*{.3cm}
\caption{Mean pixel-wise euclidean distance (and standard deviation) from the generated samples to their nearest neighbor in the training set.}
\label{tbl:nndist}
\end{table}

\paragraph{Holdout Likelihood.} We sample a set of images from a holdout test dataset and compute their latent representation using Generator Reversal. We then evaluate the likelihood of the latent codes under the distribution obtained from KDE over the latent space. The results shown in Figure~\ref{fig:bw} demonstrate that our approach achieves a very similar holdout likelihood compared to the GAN model.
The fact that data that has not been seen during training is assigned substantial likelihood is evidence for the generalization capacity of our approach and indicates that it is not a simple memorization of the training data.
An interesting direction for future work would be to investigate the role of the kernel bandwidth parameter in the tradeoff between memorization and generalization.

\begin{figure*}[h!]
\begin{center}
\centerline{
    \includegraphics[width=0.35\columnwidth]{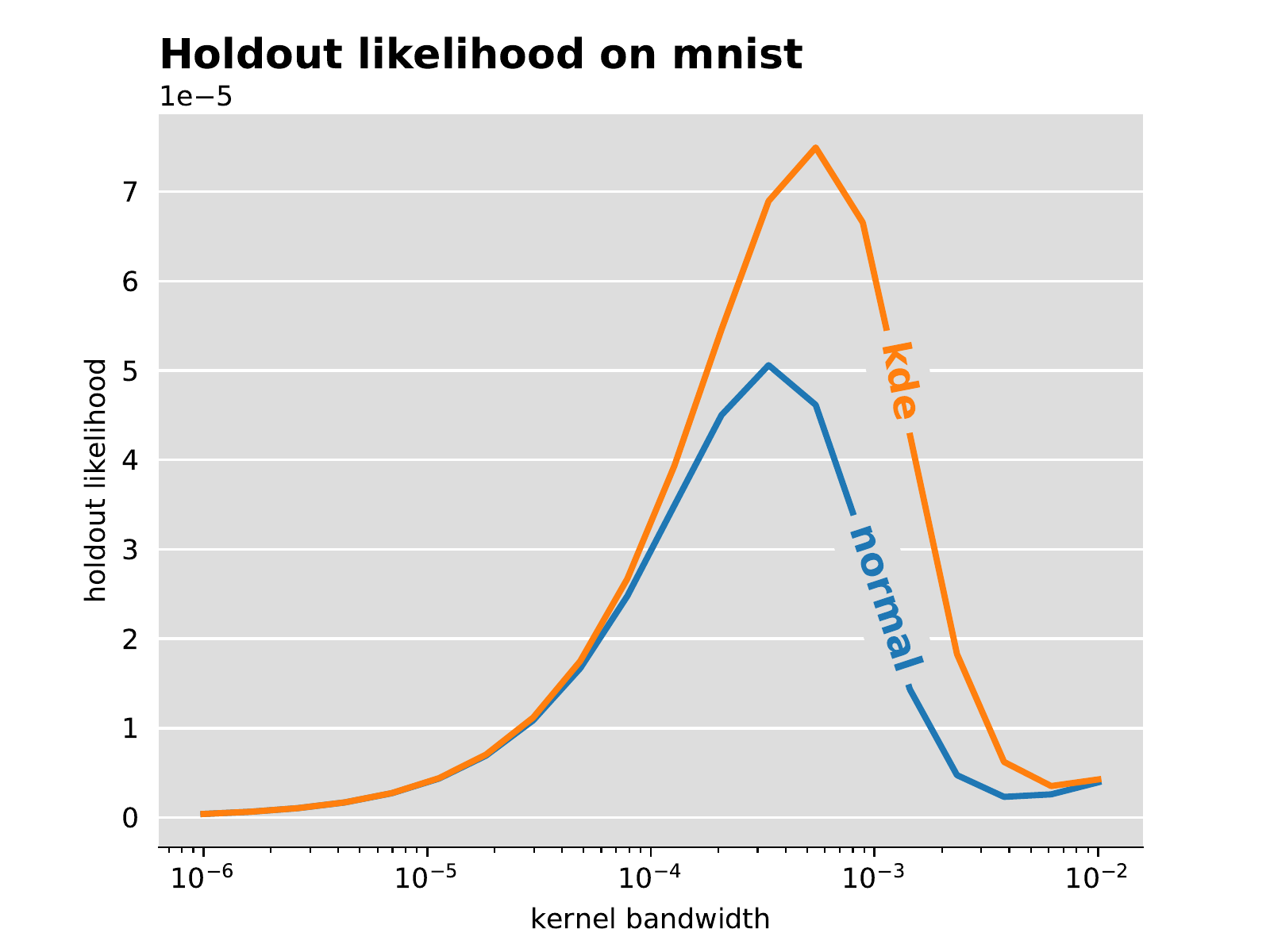}
    \includegraphics[width=0.35\columnwidth]{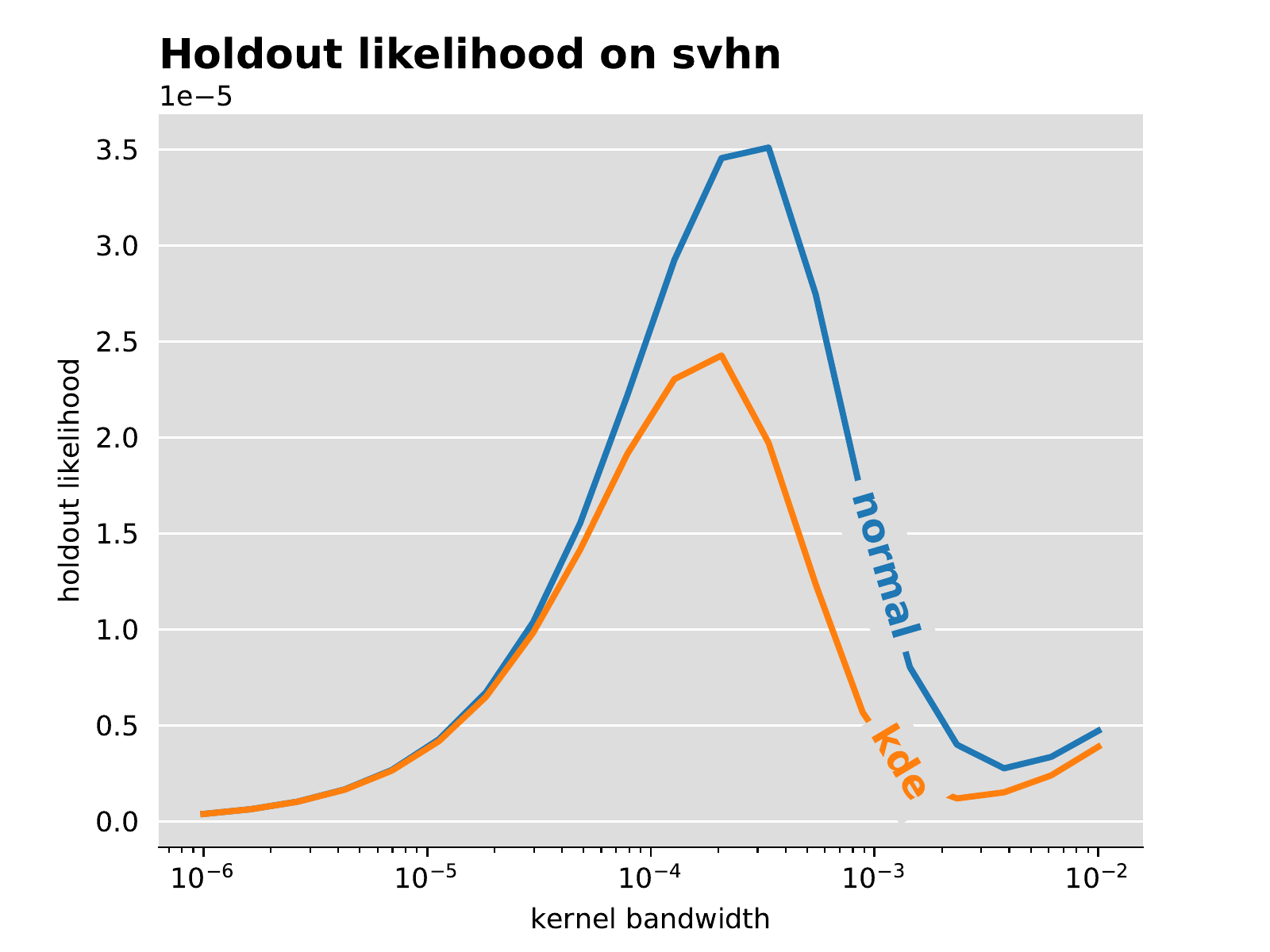}
    \includegraphics[width=0.35\columnwidth]{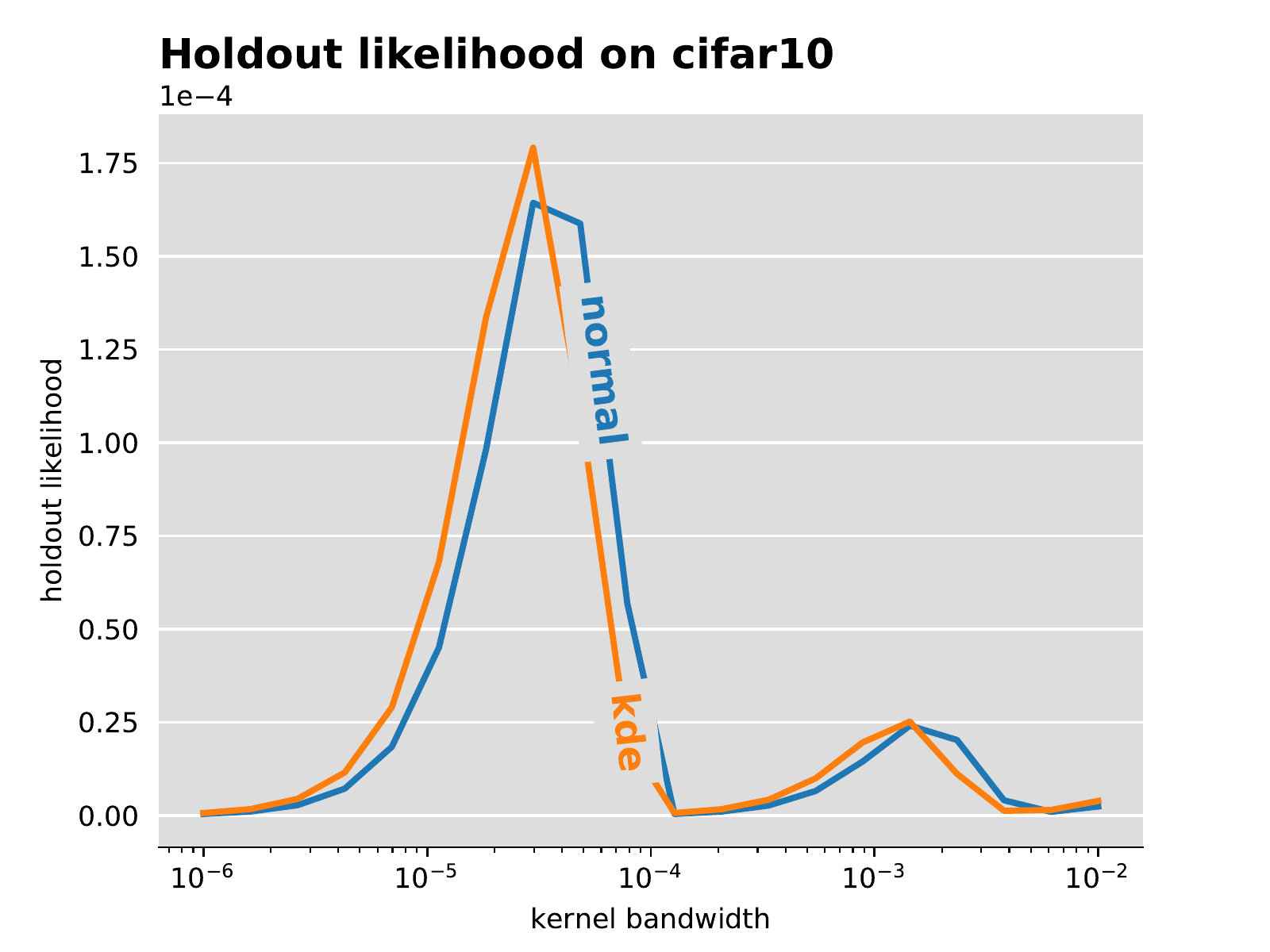}
}
\caption{\small{\it Holdout likelihood of a test set on MNIST, SVHN and CIFAR10, computed in the latent code space, while varying the kernel bandwidth used in the KDE computation.}}
\label{fig:bw}
\end{center}
\end{figure*}

\clearpage
\section{Additional Experimental Results}

\paragraph{Recovery of dropped modes.}

We create the \emph{MNIST1000} dataset by randomly concatenating 3 digits from the MNIST dataset.
This yields a dataset with 1000 modes.
To make the task more challenging, we bias sampling in favor of the high digits in a linear fashion (a 9 is 20 times more likely to be sampled than a 1).
We would like to compare the different training procedures and their ability to recover from mode dropping.
We simulate mode dropping by first restricting the dataset to samples that only consist of digits 6, 7, 8 and 9.
We train the model until convergence on this limited dataset, then we provide the full dataset and continue training, again until convergence.
The final inception score is a measure of how well each model is able to recover previously dropped modes.
We compare Generator Reversal to Vanilla GAN as well as to BiGAN\cite{donahue2016adversarial}, since we hypothesize that the inclusion of an encoder network might help recover from mode dropping.
The results are shown in Table~\ref{tbl:moderec}.
As can be seen, all models are able to recover from the initial mode dropping, but both Vanilla GAN and BiGAN are not able to reach the same score compared to when they are trained on the full dataset from the beginning.
This is an indication that some of the dropped modes have not been recovered.
In comparison, Generator Reversal is able to reach a comparable inception score as when trained on the full dataset, which indicates that it can successfully recover from mode dropping.

\begin{table}[h!]
    \centering
\begin{tabular}{l | r r r}
    & Restricted Dataset & Full Dataset after restriction  & Full Dataset\\
    \hline
    Vanilla GAN & 107 & 196 & 248\\
    BiGAN & 102 & 200 & 265 \\
    Generator Reversal & 99 & 289 & 291 \\
\end{tabular}
    \vspace*{.3cm}
\caption{Inception scores of GAN, BiGAN and GAN trained using Generator Reversal on the MNIST1000 dataset after restriction to only 3-digit numbers made up of digits 6, 7, 8 and 9. Scores on the full dataset are provided for comparison.}
\label{tbl:moderec}
\end{table}

\clearpage
\paragraph{More Trained Samples.}

Figures~\ref{fig:samples_full_model_0},~\ref{fig:samples_full_model_1},~\ref{fig:samples_full_model_2},~\ref{fig:samples_full_model_3} and~\ref{fig:samples_full_model_4} show more samples from the fully trained models of vanilla DCGAN and our KDE GAN.

\def\numfinal{0,1,2,3,4}
\foreach \nfinal in \numfinal {
\begin{center}
\begin{figure}[h!]
\foreach \ds in \datasetsa {
\begin{subfigure}[p]{0.22\textwidth}
    \includegraphics[width=1.1\columnwidth]{exp/kde/kde_\ds_total_\nfinal}
    \caption{\ds}
    \label{fig:final:\ds_\nfinal}
\end{subfigure}\quad
}
    \caption{\small{{\it Dataset images (left) and samples from fully trained vanilla DCGAN (middle) and KDE GAN (right).}}}
\label{fig:samples_full_model_\nfinal}
\end{figure}
\end{center}
}

\clearpage

\paragraph{More Beginning-Of-Training Samples.}

Figures~\ref{fig:samples_progress_0},~\ref{fig:samples_progress_1},~\ref{fig:samples_progress_2} and~\ref{fig:samples_progress_3} show more samples from the fully trained models of vanilla DCGAN and our KDE GAN.

\def\numprog{0,1,2,3}
\foreach \nprog in \numprog {
\begin{center}
\begin{figure}[h!]
\foreach \ds in \datasetsa {
\begin{subfigure}[p]{0.49\textwidth}
    \includegraphics[width=1.0\columnwidth]{exp/prog/prog_total_\ds_\nprog}
    \caption{\ds}
    \label{fig:prog:\ds_\nprog}
\end{subfigure}\quad
}
\caption{\small{{\it Image samples from the generative model at the beginning of training. Vanilla DCGAN samples on the left, our KDE GAN on the right. Each row is sampled after 50 steps of training, starting at step 100.}}}
\label{fig:samples_progress_\nprog}
\end{figure}
\end{center}
}

\clearpage

\paragraph{More Manifold Traversals.}

We perform more manifold traversal experiments and show the results in Figure~\ref{fig:traversal:kde}.

\def\numtrav{8,9,10,11,12,13,14,15,16,17,18,19,20,21,22,23,24,25,26,27,28,29,30,31,33,34,35,36,37,38,39,40}
\begin{figure*}[h!]
\begin{center}
\foreach \ntrav in \numtrav{
\centerline{
    \includegraphics[width=1.05\columnwidth]{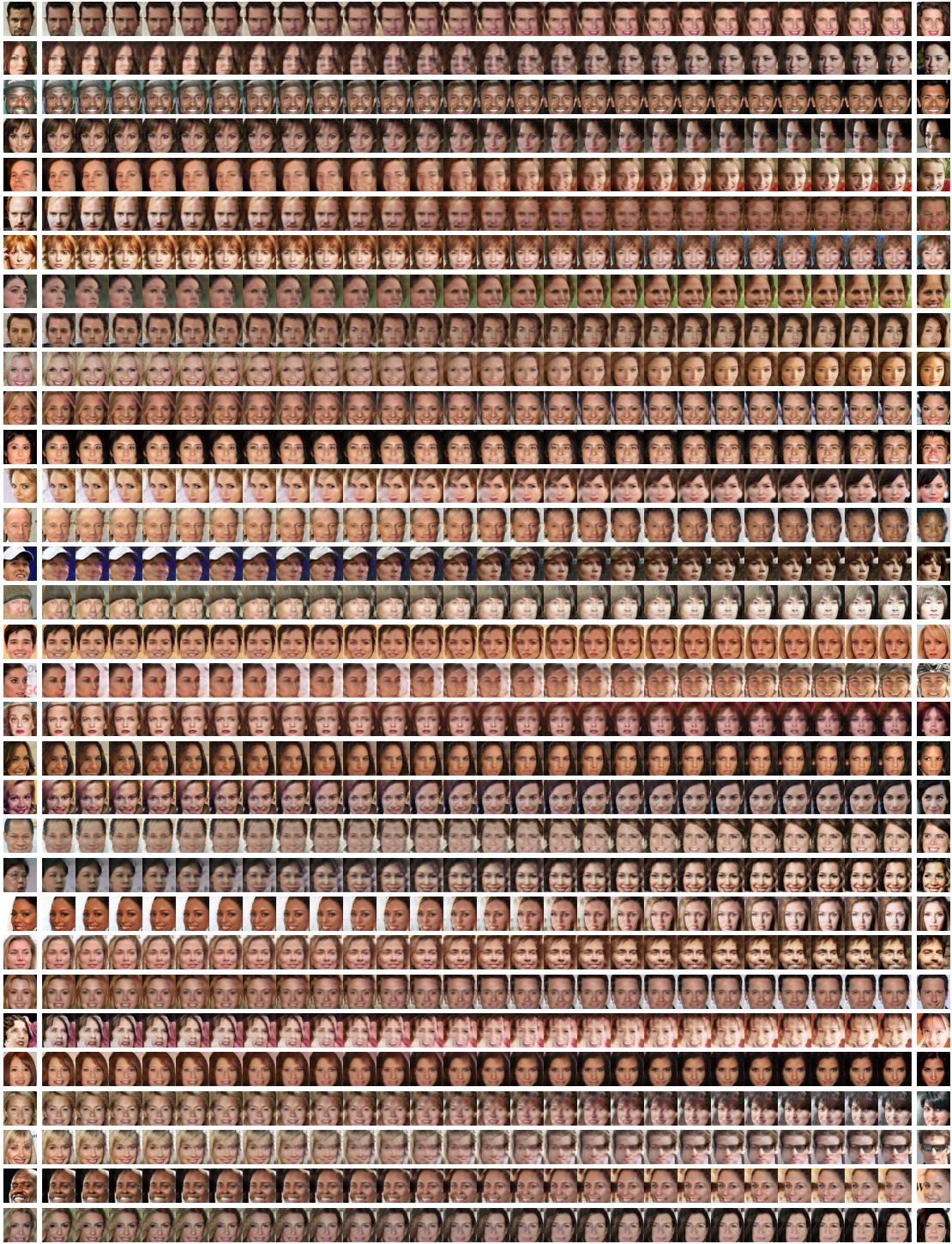}
}
    \vspace*{0.05cm}
}
    \caption{\small{{\it Manifold traversal with KDE GAN. We linearly interpolate between the latent codes of two given seeding images (far left and right). 
    }}}
\label{fig:traversal:kde}
\end{center}
\end{figure*} 

\clearpage

\paragraph{Latent Neighborhood Exploration.}

We now perform an experiment similar to manifold traversal but targeted at checking the local structure of the latent manifold around a datapoint. This is achieved by picking a random seeding image $x_0$ from the dataset, obtaining its latent representation $z_0$ via generator reversal and generating images by sampling from increasing neighborhood sizes $r>0$ centered at $\x_0$, i.e. $\z \sim B_r(\z_0) = \{ \y | \| \y - \z_0 \| < r\}$. The results give more evidence of the ability of the KDE GAN model to learn a neighborhood structure. Displayed are images for low $r$ in Figure~\ref{fig:wiggle:kde:low}, medium $r$ in Figure~\ref{fig:wiggle:kde:med}, high $r$ in Figure~\ref{fig:wiggle:kde:high} and overly high $r$ in Figure~\ref{fig:wiggle:kde:comic}.

\begin{figure*}[h!]
\begin{center}
\centerline{
    \includegraphics[width=1.0\columnwidth]{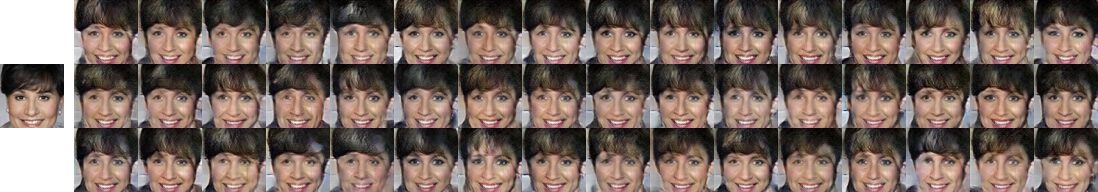}
}
    \vspace*{0.1cm}
\centerline{
    \includegraphics[width=1.0\columnwidth]{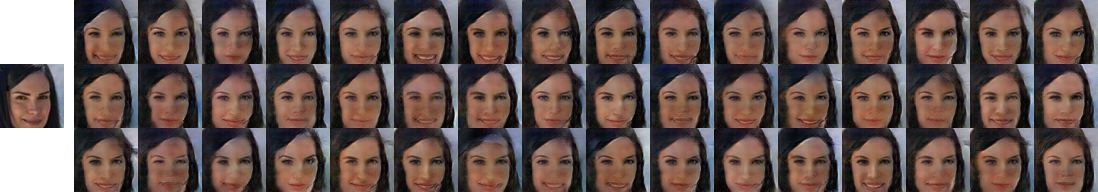}
}
    \vspace*{0.1cm}
\centerline{
    \includegraphics[width=1.0\columnwidth]{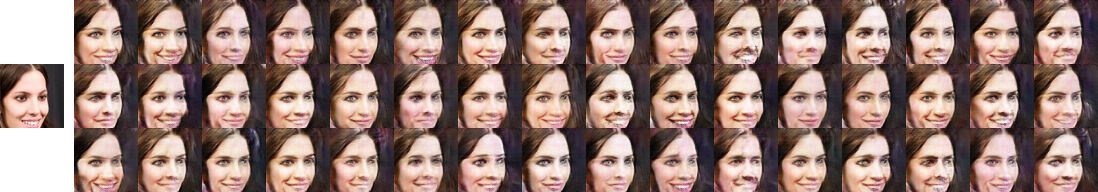}
}
    \vspace*{0.1cm}
\centerline{
    \includegraphics[width=1.0\columnwidth]{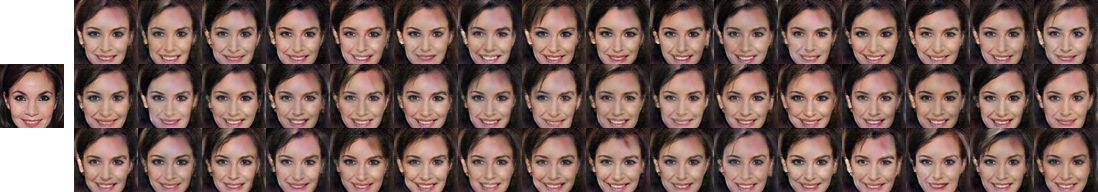}
}
    \caption{\small{\it{KDE GAN neighborhood exploration for a small neighborhood around a seeding image (show on the left).}}}
\label{fig:wiggle:kde:low}
\end{center}
\end{figure*} 

\begin{figure*}[h!]
\begin{center}
\centerline{
    \includegraphics[width=1.0\columnwidth]{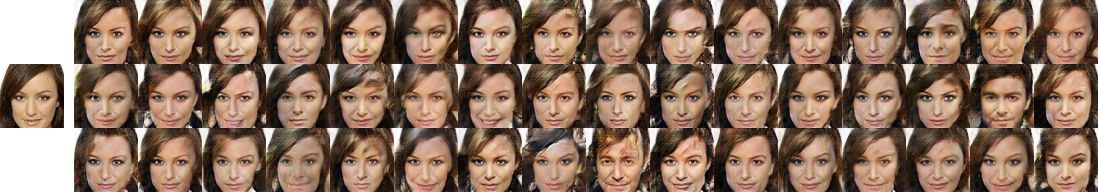}
}
\centerline{
    \includegraphics[width=1.0\columnwidth]{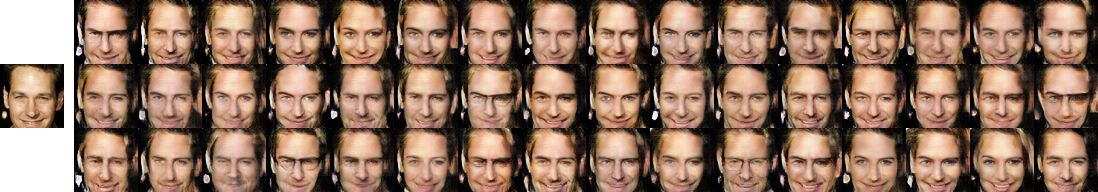}
}
\centerline{
    \includegraphics[width=1.0\columnwidth]{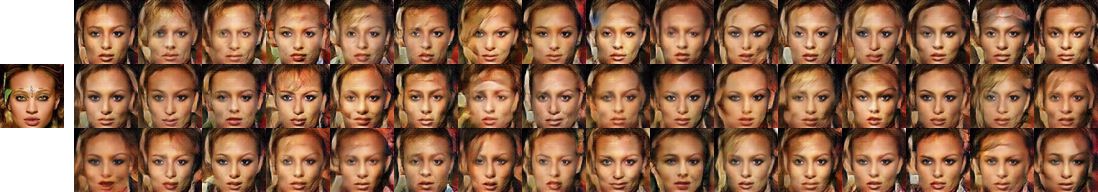}
}
\centerline{
    \includegraphics[width=1.0\columnwidth]{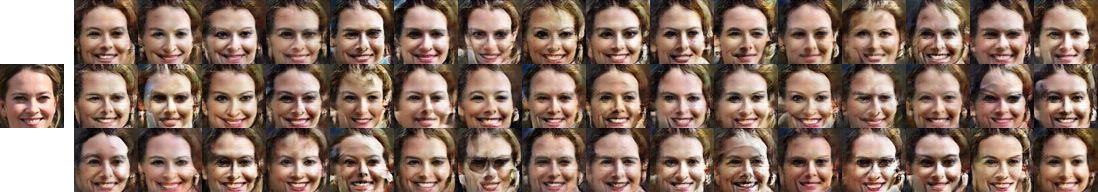}
}
    \caption{\small{\it{KDE GAN neighborhood exploration for a medium sized neighborhood around a seeding image (show on the left).}}}
\label{fig:wiggle:kde:med}
\end{center}
\end{figure*} 

\begin{figure*}[h!]
\begin{center}
\centerline{
    \includegraphics[width=1.0\columnwidth]{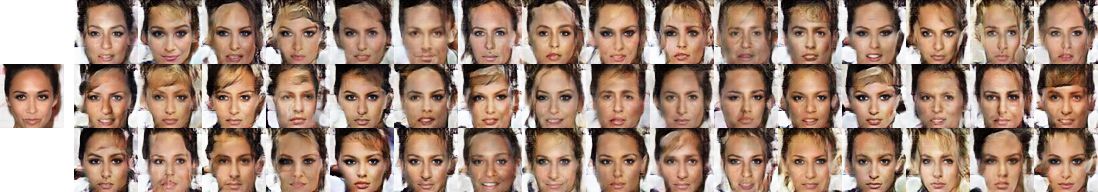}
}
\centerline{
    \includegraphics[width=1.0\columnwidth]{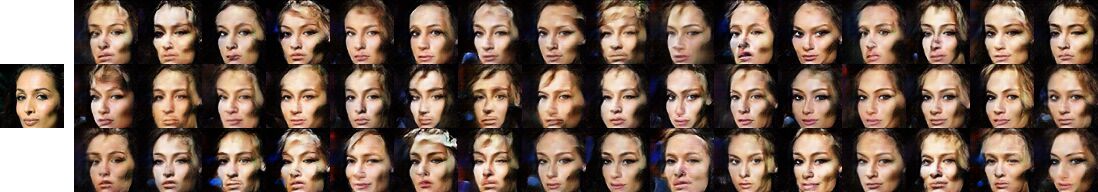}
}
\centerline{
    \includegraphics[width=1.0\columnwidth]{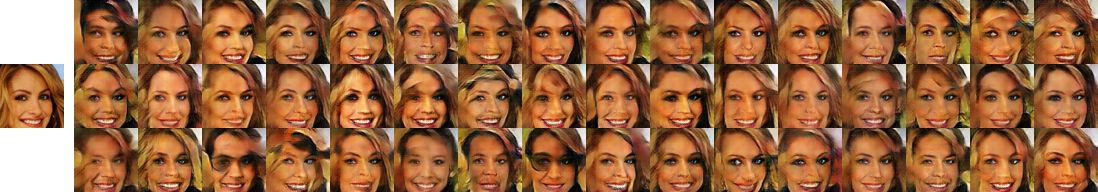}
}
\centerline{
    \includegraphics[width=1.0\columnwidth]{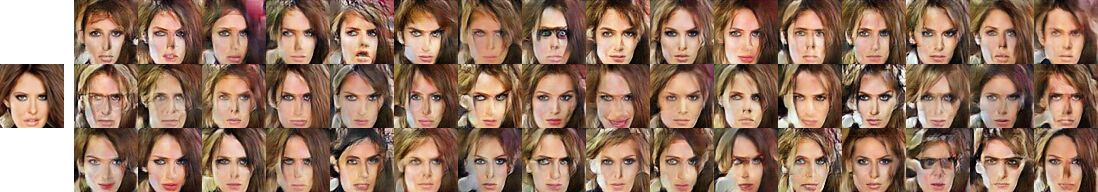}
}
    \caption{\small{\it{KDE GAN neighborhood exploration for a large neighborhood around a seeding image (show on the left).}}}
\label{fig:wiggle:kde:high}
\end{center}
\end{figure*} 

\begin{figure*}[h!]
\begin{center}
\centerline{
    \includegraphics[width=1.0\columnwidth]{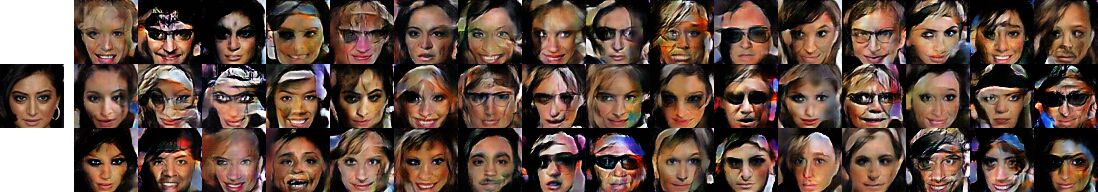}
}
\centerline{
    \includegraphics[width=1.0\columnwidth]{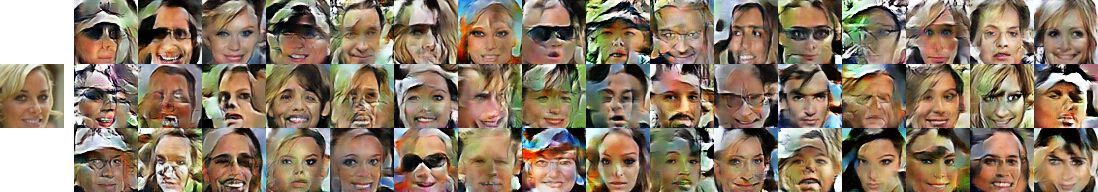}
}
\centerline{
    \includegraphics[width=1.0\columnwidth]{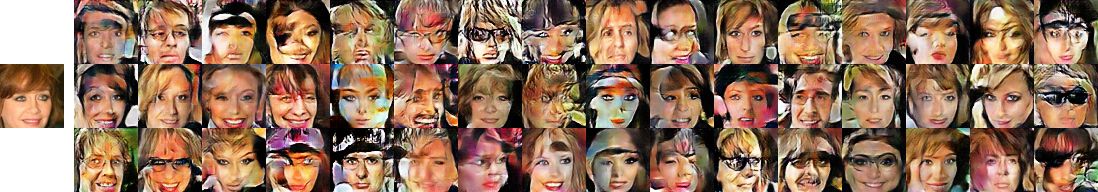}
}
\centerline{
    \includegraphics[width=1.0\columnwidth]{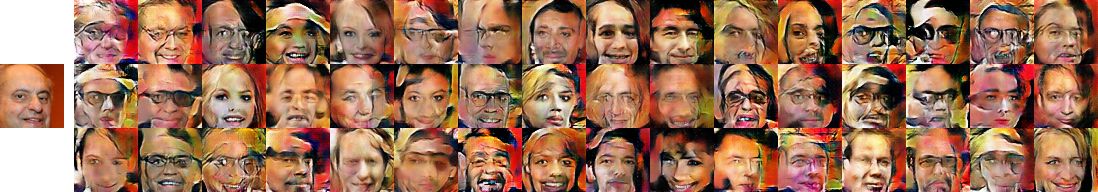}
}
    \caption{\small{\it{KDE GAN neighborhood exploration for an overly large neighborhood around a seeding image (show on the left).}}}
\label{fig:wiggle:kde:comic}
\end{center}
\end{figure*} 

\end{document}